%% file: main.tex
\documentclass{article}

\usepackage{microtype}
\usepackage{graphicx}
\usepackage{subcaption}
\usepackage{booktabs} % for professional tables

\usepackage{hyperref}

\usepackage[preprint]{icml2026}

\usepackage{amsmath}
\usepackage{amssymb}
\usepackage{mathtools}
\usepackage{amsthm}

\usepackage{bm}

\usepackage[capitalize,noabbrev]{cleveref}

\theoremstyle{plain}
\newtheorem{theorem}{Theorem}[section]
\newtheorem{proposition}[theorem]{Proposition}

\newtheorem{corollary}[theorem]{Corollary}
\theoremstyle{definition}

\newtheorem{assumption}[theorem]{Assumption}
\theoremstyle{remark}
\newtheorem{remark}[theorem]{Remark}

\usepackage[textsize=tiny]{todonotes}

\icmltitlerunning{Generative Modeling through Koopman Spectral Analysis: An Operator-Theoretic Perspective}

\begin{document}

\twocolumn[
  \icmltitle{Generative Modeling through Koopman Spectral Analysis: \\
    An Operator-Theoretic Perspective}



  \icmlsetsymbol{equal}{*}

  \begin{icmlauthorlist}
    \icmlauthor{Yuanchao Xu}{equal,kyotou}
    \icmlauthor{Fengyi Li}{equal,LinkedIn}
    \icmlauthor{Masahiro Fujisawa}{ou,riken}
    \icmlauthor{Xiaoyuan Cheng}{ucl}
    \icmlauthor{Youssef Marzouk}{mit}
    \icmlauthor{Isao Ishikawa}{kyotou}
  \end{icmlauthorlist}

  \icmlaffiliation{kyotou}{Center for Science Adventure and Collaborative Research Advancement, Graduate School of Science, Kyoto University, Kyoto, Japan}
  \icmlaffiliation{mit}{Massachusetts Institute of Technology, Cambridge, MA}
  \icmlaffiliation{LinkedIn}{LinkedIn Corporation, Sunnyvale, CA}
  \icmlaffiliation{ucl}{Dynamic Systems Lab, University College London, London, UK}
  \icmlaffiliation{ou}{Graduate School of Information Science and Technology, The University of Osaka, Osaka, Japan}
  \icmlaffiliation{riken}{RIKEN Center for Advanced Intelligence Project, Tokyo, Japan}

  \icmlcorrespondingauthor{Yuanchao Xu}{xu.yuanchao.3a@kyoto-u.ac.jp}

  \icmlkeywords{Machine Learning, ICML}

  \vskip 0.3in
]



\printAffiliationsAndNotice{}  % no special notice (required even if empty)

\begin{abstract}
We propose Koopman Spectral Wasserstein Gradient Descent (KSWGD), a particle-based generative modeling framework that learns the Langevin generator via Koopman theory and integrates it with Wasserstein gradient descent. Our key insight is that this spectral structure of the underlying distribution can be directly estimated from trajectory data via the Koopman operator, eliminating the need for explicit knowledge of the target potential. Additionally, we prove that KSWGD maintains an approximately constant dissipation rate, thereby establishing linear convergence and overcoming the vanishing-gradient phenomenon that hinders existing kernel-based particle methods. We further provide a Feynman--Kac interpretation that clarifies the method's probabilistic foundation. Experiments on compact manifolds, metastable multi-well systems, and high-dimensional stochastic partial differential equations demonstrate that KSWGD consistently outperforms baselines in both convergence speed and sample quality.
\end{abstract}

%%%%%%%% Section 1 %%%%%%%%
\section{Introduction}
Generative modeling plays a central role in modern data science, aiming to generate new samples from an unknown target distribution without explicit access to its density. While approaches like variational autoencoders, generative adversarial networks, and diffusion models have achieved striking empirical success~\citep{goodfellow2014generative,kingma2013auto,ho2020denoising}, they often demand extensive neural network training and delicate hyperparameter tuning. Furthermore, score-based methods rely on estimating gradients of the log-density, which can be computationally expensive or unstable in high-dimensional scientific problems.

A complementary line of work formulates generative modeling as Wasserstein gradient flows on the space of probability measures~\citep{jordan1998variational,ambrosio2008gradient}. This variational perspective underpins particle-based algorithms such as Stein variational gradient descent (SVGD)~\citep{liu2016stein} and Laplacian-adjusted Wasserstein gradient descent (LAWGD)~\citep{chewi2020svgd}. LAWGD, in particular, achieves strong theoretical guarantees by preconditioning the gradient flow with the inverse Langevin generator $\mathcal{L}^{-1}$, thereby equalizing scales across stiff modes and accelerating convergence. However, realizing this ideal preconditioner is challenging in practice, as access to $\mathcal{L}^{-1}$ typically requires explicit knowledge of the potential function, which is unavailable in many generative settings.

Recent data-driven approaches, such as the diffusion map particle system (DMPS)~\citep{li2025diffusion}, approximate the generator from samples via diffusion maps. While training-free, this method rely on static pairwise distances between i.i.d.\ samples and often require careful bandwidth tuning. Crucially, this approach does not explicitly leverage temporal trajectory information, and its kernel-based spectral approximation does not scale well to high-dimensional settings.

In this work, we introduce an \emph{operator-theoretic perspective} that bridges this gap. We leverage the fundamental connection that the Langevin generator $\mathcal{L}$ coincides (up to a sign) with the infinitesimal generator of the Koopman operator associated with the overdamped Langevin dynamics. This insight allows us to approximate the spectrum of $\mathcal{L}$ directly from trajectory data using Koopman spectral methods~\citep{Williams_2015,kevrekidis2016kernel}, rather than relying on static geometry. By truncating this data-driven spectral decomposition, we construct a finite-rank approximation of $\mathcal{L}^{-1}$ that is naturally informed by the system's temporal evolution.

Based on this connection, we propose \textit{Koopman Spectral Wasserstein Gradient Descent} (KSWGD), a particle-based framework that integrates Wasserstein gradient flows with modern Koopman theory developed by~\citet{mezic2005spectral,mezic2021koopman}.
KSWGD replaces the inaccessible preconditioner in LAWGD with a Koopman-based spectral approximation. This leads to deterministic particle dynamics with an approximately constant dissipation rate, effectively mitigating the vanishing-gradient issues plaguing existing kernel methods. The resulting algorithm is training-free once the Koopman basis is estimated and accommodates both time-series and static data.

Our contributions are threefold. First, we provide a rigorous convergence analysis, establishing linear convergence in discrete time, with explicit error bounds separating spectral truncation from estimation error. Second, we establish a connection to Feynman--Kac theory~\citep{oksendal2003stochastic}, clarifying the method's probabilistic foundation. Finally, empirical results on compact manifolds, metastable multi-well systems, and high-dimensional stochastic partial differential equations demonstrate that KSWGD consistently achieves faster convergence than existing particle methods while maintaining high sample quality.

%%%%%%%% Section 2 %%%%%%%%
\section{Preliminaries}
% Preliminary: Background on Wasserstein Gradient Flow and Koopman Operator Theory
This section establishes the theoretical foundations for our method: the geometry of Wasserstein gradient flows in Section~\ref{sec:preliminary1} and the spectral properties of the Koopman operator in Section~\ref{sec:koopman}.

%%%%%%%% Section 2.1 %%%%%%%%
\subsection{Wasserstein Gradient Flow}\label{sec:preliminary1}

%Let $\pi$ be a target distribution on $\mathbb{R}^d$ with density $\pi(x) \propto e^{-V(x)}$. The Jordan--%Kinderlehrer--Otto (JKO) framework \citep{jordan1998variational,otto2001geometry}
%interprets evolution equations for probability measures as Wasserstein gradient flows on
%\(
%(\mathcal{P}_2(\mathbb{R}^d), W_2),
%\)
%where $\mathcal{P}_2(\mathbb{R}^d)$ denotes probability measures with finite second moment and
%$W_2$ is the $2$-Wasserstein distance \citep{santambrogio2015optimal}.
%Given a functional $\mathcal{F}:\mathcal{P}_2(\mathbb{R}^d)\to\mathbb{R}$, the associated
%Wasserstein gradient flow can be written as the continuity equation
%\[
%\partial_t \mu_t + \operatorname{div}(\mu_t\, v_t)=0,
%\qquad
%v_t(x) = - \nabla_{W_2}\mathcal{F}(\mu_t)(x),
%\]
%where $\nabla_{W_2}\mathcal{F}(\mu):\mathbb{R}^d\to\mathbb{R}^d$ denotes the Wasserstein gradient
%\citep{santambrogio2015optimal}. 
%Throughout, we assume $\mu \ll \pi$ and write the density ratio $\rho \coloneqq \frac{\mathrm{d}\mu}\mathrm{d}\pi}.$
% \begin{equation*}
%     \rho \coloneqq \frac{\mathrm{d}\mu}{\mathrm{d}\pi}.
% \end{equation*}

Let $\pi(x) \propto e^{-V(x)}$ be a target density on $\mathbb{R}^d$. The Jordan--Kinderlehrer--Otto (JKO) framework \citep{jordan1998variational,otto2001geometry} interprets the evolution of probability measures as a gradient flow in the Wasserstein space $(\mathcal{P}_2(\mathbb{R}^d), W_2)$, where $W_2$ denotes the $2$-Wasserstein distance.
The flow follows the continuity equation
\begin{equation*}
    \partial_t \mu_t + \operatorname{div}(\mu_t\, v_t)=0
\end{equation*}
driven by the velocity field \citep{santambrogio2015optimal}
\begin{equation*}
    v_t = -\nabla_{W_2}\mathcal{F}(\mu_t).
\end{equation*}
%Two classical choices of $\mathcal{F}$ are the Kullback--Leibler (KL) divergence and the
%chi-squared divergence:
%\begin{equation*}
%\mathrm{KL}(\mu\|\pi)
%\coloneqq \int_{\mathbb{R}^d}\log\!\left(\frac{\mathrm{d}\mu}{\mathrm{d}\pi}\right)\,\mathrm{d}\mu
%= \int_{\mathbb{R}^d} \rho \log\rho \,\mathrm{d}\pi,
%\end{equation*}
%\begin{equation*}
%\chi^2(\mu\|\pi)
%\coloneqq \mathbb{E}_\pi\!\left[(\rho-1)^2\right]
%= \int_{\mathbb{R}^d} (\rho-1)^2 \,\mathrm{d}\pi.
%\end{equation*}
%Their Wasserstein gradients are given by \citep{ambrosio2008gradient,santambrogio2015optimal,Tsybakov2009}
%\[
%\nabla_{W_2}\mathrm{KL}(\mu) = \nabla \log\rho,
%\qquad
%\nabla_{W_2}\chi^2(\mu) = 2\nabla \rho.
%\]
%In particular, the $\chi^2$-gradient flow takes the explicit form
%\[
%\partial_t \mu_t
%=
%2\,\operatorname{div}\!\left(\mu_t \nabla \rho_t\right),
%\]
%where $\rho_t \coloneqq \frac{\mathrm{d}\mu_t}{\mathrm{d}\pi}$, which enables sharper convergence analysis for %deterministic particle methods \citep{chewi2020svgd}.

Let $\mu \ll \pi$ and write the density ratio $\rho \coloneqq \frac{\mathrm{d}\mu}{\mathrm{d}\pi}.$
We then focus on the $\chi^2$-divergence $\chi^2(\mu\|\pi) \coloneqq \int (\rho-1)^2 \,\mathrm{d}\pi$, which provides sharper convergence guarantees than the KL divergence~\citep{chewi2020svgd}.
Its Wasserstein gradient is
\begin{equation*}
    \nabla_{W_2}\chi^2(\mu_t) = 2\nabla \rho_t,
\end{equation*}
leading to the flow 
\begin{equation*}
    \partial_t \mu_t = 2\operatorname{div}(\mu_t \nabla \rho_t).
\end{equation*}
Direct simulation of this flow is computationally intractable due to the difficulty of estimating the density ratio $\rho_t = \frac{\mathrm{d}\mu_t}{\mathrm{d}\pi}$ and its gradients in high dimensions.

%Although the $\chi^2$-gradient flow admits strong theoretical guarantees, it is often computationally intractable in %practice
%because it requires evaluating (or approximating) the density ratio $\rho_t$ at each iteration.
%A natural remedy is to introduce a spectral preconditioner through an integral operator $\mathcal{K}_\pi$, which in %the ideal case acts diagonally on the eigenmodes of the Langevin generator and rescales its eigenvalues, leading to %the following preconditioned flow
%\begin{equation}
%\partial_t \mu_t
%=
%\operatorname{div}\!\left(\mu_t \,\nabla \mathcal{K}_\pi \rho_t\right),
%\label{eq:preconditioned_flow}
%\end{equation}
%where the operator $\mathcal{K}_\pi$ is defined by
%\[
%(\mathcal{K}_\pi f)(x)
%\coloneqq
%\int_{\mathbb{R}^d} K(x,y)\, f(y)\, \mathrm{d}\pi(y),
%\]
%for a kernel $K:\mathbb{R}^d\times\mathbb{R}^d\to\mathbb{R}$.
%The key challenge is to construct an effective preconditioner $\mathcal{K}_\pi$ that accelerates convergence while %remaining
%computationally tractable. We defer the detailed construction to Section~\ref{sec:method}, where we develop a data-%driven
%approach based on Koopman spectral theory.

To overcome this, we employ spectral preconditioning via an integral operator $\mathcal{K}_\pi$. Ideally, setting $\mathcal{K}_\pi$ to the inverse Langevin generator $\mathcal{L}^{-1}$ acts as a Newton-type update, equalizing spectral scales. The preconditioned flow is given by:
\begin{equation}
\partial_t \mu_t
=
\operatorname{div}\!\left(\mu_t \,\nabla \mathcal{K}_\pi \rho_t\right),
\label{eq:preconditioned_flow}
\end{equation}
where $(\mathcal{K}_\pi f)(x) \coloneqq \int K(x,y) f(y) \mathrm{d}\pi(y)$ for a kernel $K$.
While the ideal choice $K_{\mathcal{L}^{-1}}$ is typically unknown, our goal in Section~\ref{sec:method} is to approximate it purely from data.

%\textbf{Particle implementation.}
%Equation~\eqref{eq:preconditioned_flow} corresponds to the velocity field
%\[
%v_t(\cdot) = -\nabla \mathcal{K}_\pi \rho_t(\cdot).
%\]
%Using $\rho_t\,\mathrm{d}\pi = \mathrm{d}\mu_t$, we have the identity
%\begin{align*}
%\nabla \mathcal{K}_\pi \rho_t(x)
%&= \int_{\mathbb{R}^d} \nabla_1 K(x,y)\, \rho_t(y)\, \mathrm{d}\pi(y) \\
%&= \int_{\mathbb{R}^d} \nabla_1 K(x,y)\, \mathrm{d}\mu_t(y),
%\end{align*}
%where $\nabla_1$ denotes the gradient with respect to the first argument.
%Approximating $\mu_t$ by the empirical measure
%\(
%\mu_t \approx \frac{1}{M}\sum_{j=1}^M \delta_{x_t^{(j)}},
%\)
%the induced particle dynamics become
%\begin{equation}\label{eq:lawgd_discrete}
%\dot{x}_t^{(i)}
%=
%-\frac{1}{M} \sum_{j=1}^M \nabla_1 K_{\mathcal{L}^{-1}}\!\left(x_t^{(i)}, x_t^{(j)}\right),
%\qquad 1 \leq i \leq M,
%\end{equation}
%where $K_{\mathcal{L}^{-1}}$ denotes the kernel associated with the (ideal) choice of preconditioner %$\mathcal{K}_\pi=\mathcal{L}^{-1}$
%as in LAWGD (its data-driven approximation will be constructed in Section~\ref{sec:method}).

\paragraph{Particle implementation.}
The particle dynamics are derived from the velocity field
\begin{equation*}
    v_t(\cdot) = -\nabla \mathcal{K}_\pi \rho_t(\cdot).
\end{equation*}
Crucially, the dependence on the density ratio $\rho_t$ can be eliminated using the identity $\rho_t(y)\mathrm{d}\pi(y) = \mathrm{d}\mu_t(y)$:
\begin{align*}
\nabla \mathcal{K}_\pi \rho_t(x)
&= \int_{\mathbb{R}^d} \nabla_1 K(x,y)\, \rho_t(y)\, \mathrm{d}\pi(y) \\
&= \int_{\mathbb{R}^d} \nabla_1 K(x,y)\, \mathrm{d}\mu_t(y).
\end{align*}
This allows us to evaluate the gradient solely using samples from $\mu_t$.
Approximating $\mu_t \approx \frac{1}{M}\sum_{j=1}^M \delta_{x_t^{(j)}}$, we obtain the computable particle update:
\begin{equation}\label{eq:lawgd_discrete}
\dot{x}_t^{(i)}
=
-\frac{1}{M} \sum_{j=1}^M \nabla_1 K_{\mathcal{L}^{-1}}\!\left(x_t^{(i)}, x_t^{(j)}\right),
\quad 1 \leq i \leq M.
\end{equation}
In Section~\ref{sec:method}, we will construct a data-driven approximation of $K_{\mathcal{L}^{-1}}$ using Koopman spectral analysis.

%%%%%%%% Section 2.2 %%%%%%%%
\subsection{Koopman Operator}\label{sec:koopman}
%Consider a dynamical system governed by the SDE
%\begin{equation}\label{eq:sde}
%\mathrm{d}X_t = b(X_t) \mathrm{d}t + \sigma(X_t) \mathrm{d}W_t,
%\end{equation}
%where $b: \mathbb{R}^d \to \mathbb{R}^d$ is the drift, $\sigma: \mathbb{R}^d \to \mathbb{R}^{d \times m}$ is the %diffusion coefficient satisfying appropriate regularity conditions \cite{engel1999one}, and $W_t$ is an $m$-%dimensional Brownian motion.
Consider a dynamical system on $\mathbb{R}^d$ governed by the SDE
\begin{equation}\label{eq:sde}
\mathrm{d}X_t = b(X_t) \mathrm{d}t + \sigma(X_t) \mathrm{d}W_t,
\end{equation}
where $b: \mathbb{R}^d \to \mathbb{R}^d$ is the drift, $\sigma: \mathbb{R}^d \to \mathbb{R}^{d \times m}$ is the diffusion coefficient, and $W_t$ is a standard $m$-dimensional Brownian motion.
Koopman operator theory \citep{Koopman1931, koopman1932dynamical, mezic2005spectral} lifts this nonlinear dynamics to a linear operator $\mathcal{T}^t$ acting on observable functions $g \in L^2(\pi)$ via conditional expectations:
\begin{equation}
(\mathcal{T}^t g)(x) \coloneqq \mathbb{E}[g(X_t) \mid X_0 = x].
\end{equation}
Crucially, $\mathcal{T}^t$ is linear even if the underlying SDE is nonlinear, enabling the use of spectral methods.
The evolution of observables is governed by the infinitesimal generator $\mathcal{A}$, defined as
\begin{equation*}
    \mathcal{A}f \coloneqq \lim_{t \to 0} \frac{\mathcal{T}^t f - f}{t}
\end{equation*}
on its domain $\mathcal{D}(\mathcal{A}) \subset L^2(\pi)$.
By It\^{o}'s lemma~\citep{engel1999one}, this generator admits the explicit second-order differential form:
\begin{equation}\label{eq:Ito}
\mathcal{A} f(x) = \langle b(x), \nabla f(x) \rangle + \frac{1}{2} \text{Tr}\left(\sigma(x) \sigma(x)^\top \nabla^2 f(x)\right).
\end{equation}
This operator-theoretic perspective is powerful because the eigenpairs of $\mathcal{A}$ encode the global properties of the stochastic dynamics.

While $\mathcal{A}$ is infinite-dimensional, it can be approximated from finite trajectory data.
Extended Dynamic Mode Decomposition (EDMD)~\citep{Williams_2015} and its variants project the operator onto a finite-dimensional subspace.
Recent advances extend this to reproducing kernel Hilbert spaces~\citep{kevrekidis2016kernel, kostic2022learning, ishikawa2024koopman} and neural networks~\citep{lusch2018deep, baikonode, pmlr-v267-zhang25be, xu2025reinforced}, allowing for efficient spectral estimation even in high-dimensional stochastic settings~\citep{arbabi2017ergodic, oprea2025distributional, 10.1063/5.0283640}.

%%%%%%%% Section 3 %%%%%%%%
%\section{Data-Driven Spectral Construction}\label{sec:method}
\section{Koopman Spectral Wasserstein Gradient Descent}\label{sec:method}
In this section, we present our main framework, \emph{Koopman Spectral Wasserstein Gradient Descent} (KSWGD). 
The central goal is to construct a tractable approximation of the ideal preconditioner $\mathcal{K}_\pi = \mathcal{L}^{-1}$ solely from trajectory data, without explicit knowledge of the potential $V$. 
We achieve this by exploiting the spectral correspondence between the Langevin generator and the Koopman operator.

\subsection{Operator-Theoretic Foundation}\label{sec:foundation}

Let $\mathcal{L} = -\Delta + \langle \nabla V, \nabla \cdot \rangle$ be the Langevin generator on $L^2_\pi$. Assuming a discrete spectrum, the ideal preconditioner acts as
\begin{equation}
\mathcal{K}_\pi f = \sum_{i=1}^{\infty} \lambda_i^{-1} \langle f, \phi_i \rangle_\pi \phi_i,
\label{eq:ideal_kernel_op}
\end{equation}
where $\{(\lambda_i, \phi_i)\}$ are the eigenpairs of $\mathcal{L}$ with $0=\lambda_0 < \lambda_1 \le \dots$.
Direct computation of Eq.~\eqref{eq:ideal_kernel_op} is, as mentioned in Section~\ref{sec:preliminary1}, generally impossible due to the unknown $V$.

However, consider the Koopman operator associated with the overdamped Langevin dynamics in Eq.~\eqref{eq:sde}. Its infinitesimal generator $\mathcal{A}$ satisfies
\begin{equation*}
    \mathcal{A}f = -\langle \nabla V, \nabla f\rangle + \Delta f = -\mathcal{L}f,
\end{equation*}
on their common domain. This identity $\mathcal{A} = -\mathcal{L}$ implies that the eigenfunctions of $\mathcal{A}$ are exactly those of $\mathcal{L}$, while the eigenvalues are simply negated.
Consequently, we can approximate the spectral components of the inverse Langevin generator by estimating the spectrum of the Koopman generator from time-series data.

Defining the rank-$r$ truncation, we propose the spectral preconditioner:
\begin{equation}\label{eq:Kr_def}
\mathcal{K}_r
\coloneqq
\sum_{i=1}^r \frac{1}{\lambda_i}\,\langle \cdot, \phi_i\rangle_\pi\,\phi_i,
\end{equation}
which approximates $\mathcal{L}^{-1}$ and recovers the exact inverse as $r\to\infty$ 
\citep{10.1063/5.0283640}. This truncated operator serves as our data-driven 
preconditioner in the Wasserstein gradient flow.

\subsection{Data-Driven Spectral Approximation}\label{sec:approximation}
To realize $\mathcal{K}_r$ in practice, we employ data-driven Koopman methods such as EDMD \citep{Williams_2015} or kernel-EDMD \citep{kevrekidis2016kernel}.
Given training trajectories $\{(z^{(j)}, z'^{(j)})\}_{j=1}^N$ where $z'^{(j)}$ is the state of $z^{(j)}$ after a short interval $\tau$, these methods approximate the generator $\mathcal{A}$ on a finite-dimensional subspace spanned by a dictionary (or kernel feature map).

Solving the generalized eigenvalue problem associated with the data matrices yields estimated eigenpairs $\{(\widehat{\lambda}_i, \widehat{\phi}_i)\}_{i=1}^r$. Substituting these into \eqref{eq:Kr_def} gives our empirical preconditioner:
\begin{equation}\label{eq:Kr_hat_def}
    \widehat{\mathcal{K}}_r f(x) = \sum_{i=1}^r \frac{1}{\widehat{\lambda}_i}\langle f, \widehat{\phi}_i\rangle_\pi \widehat{\phi}_i(x).
\end{equation}
Crucially, this construction is \emph{training-free} in the sense of deep learning; it requires only linear algebra operations and avoids iterative backpropagation.

\subsection{Particle Dynamics and Algorithm}\label{sec:algorithm}
With the estimated spectral components, we formulate the particle update rule. The velocity field requires the gradient of the kernel $K_{\widehat{\mathcal{K}}_r}(x,y) = \sum_{i=1}^r \widehat{\lambda}_i^{-1} \widehat{\phi}_i(x)\widehat{\phi}_i(y)$ with respect to the first argument.
Substituting this into the particle implementation \eqref{eq:lawgd_discrete}, the update rule for the $i$-th particle $x_t^{(i)}$ is:
\begin{equation}
x_{t+1}^{(i)} 
= 
x_t^{(i)} - \frac{h}{M} \sum_{j=1}^M \sum_{k=1}^r \frac{\nabla\widehat{\phi}_k(x_t^{(i)}) \cdot \widehat{\phi}_k(x_t^{(j)})}{\widehat{\lambda}_k}.
\label{eq:kswgd_update}
\end{equation}
\textbf{Gradient Computation:} A practical challenge is computing $\nabla\widehat{\phi}_k$ in \eqref{eq:kswgd_update}, which is typically hard for general function approximators. In this work, we compute it analytically using a backward kernel matrix with a Gaussian RBF kernel (see \cite{li2025diffusion} for details).
Algorithm~\ref{alg:kswgd} summarizes the complete procedure for time-series data. Importantly, our framework extends to static data (e.g., image datasets) by constructing synthetic dynamics; we provide the corresponding procedure in Algorithm~\ref{alg:kswgd_static} (Appendix~\ref{app:s_1}).

\paragraph{Complexity:}
The computational cost is dominated by the offline eigendecomposition ($O(n^3)$ for basis size $n$) and the online kernel gradient evaluation ($O(M^2 rd)$ per iteration for $M$ particles, rank $r$, dimension $d$).
Unlike score-based generative models \citep{song2020score} that require extensive neural network training and stochastic sampling, KSWGD's sampling phase is entirely \emph{deterministic and training-free} once the Koopman basis is determined.

\begin{algorithm}[tb]
  \caption{KSWGD for Time-Series Data}
  \label{alg:kswgd}
  \begin{algorithmic}
    \STATE {\bfseries Input:} training samples $\{z^{(j)}\}_{j=1}^N \sim \pi$, initial particles $\{x_0^{(i)}\}_{i=1}^M$, step size $h$, truncation rank $r$, dictionary size $n$, max iterations $T$.
    \STATE Construct dictionary $\{\psi_k\}_{k=1}^n$ and estimate the Koopman operator using trajectory pairs $\{(z^{(j)}, z^{(j+1)})\}_{j=1}^{N-1}$.
    \STATE Compute leading $r$ eigenpairs $\{(\widehat{\lambda}_k, \widehat{\phi}_k)\}_{k=1}^r$ of $-\widehat{\mathcal{A}}_N$.
    \FOR{$t = 0$ {\bfseries to} $T-1$}
        \FOR{$i = 1$ {\bfseries to} $M$}
            \STATE Initialize $\mathbf{v}_i \gets \mathbf{0} \in \mathbb{R}^d$.
            \FOR{$j = 1$ {\bfseries to} $M$}
                \STATE $\mathbf{v}_i \gets \mathbf{v}_i + \sum_{k=1}^r 
                \frac{\nabla \widehat{\phi}_k(x_t^{(i)}) \cdot \widehat{\phi}_k(x_t^{(j)})}
                {\widehat{\lambda}_k}.$ \COMMENT{see Eq.~\eqref{eq:kswgd_update}}
            \ENDFOR
            \STATE $x_{t+1}^{(i)} \gets x_t^{(i)} - \frac{h}{M} \mathbf{v}_i.$ \COMMENT{see Eq.~\eqref{eq:kswgd_update}}
        \ENDFOR
    \ENDFOR
    \STATE {\bfseries Output:} generated particles $\{x_T^{(i)}\}_{i=1}^M$.
  \end{algorithmic}
\end{algorithm}

%%%%%%%% Section 4 %%%%%%%%
\section{Convergence and Error Bound Analysis}
This section develops the main theoretical guarantees for our proposed method KSWGD. 
We first analyze the continuous-time dynamics and derive error bounds that separate the ideal preconditioned flow from the effects of truncation and data-driven Koopman approximation in Section~\ref{sec:main_results}. 
We then extend these theory to the practical discrete-time algorithm via the Approximate Gradient Flow framework in Section~\ref{sec:discrete_agf}, which shows a linear convergence (up to an explicit bias term) under suitable step-size conditions.

%%%%%%%% Section 4.1 %%%%%%%%
\subsection{Data-Driven Error Bound Analysis}\label{sec:main_results}
In this section, we analyze the convergence properties of the Koopman spectral Wasserstein
gradient descent dynamics introduced in Section~\ref{sec:method}.  
Throughout the section, we denote $\rho_t = \frac{\mathrm{d}\mu_t}{\mathrm{d}\pi}$ and 
\begin{equation*}\label{def:density}
f_t \coloneqq \rho_t - 1 \in L_0^2(\pi),
\end{equation*}
the fluctuation of the density with respect to~$\pi$.
Let $\{(\phi_i,\lambda_i)\}_{i\ge1}$ be the eigenpairs of the Langevin generator $\mathcal{L}$,  
and recall the truncated inverse operator in \eqref{eq:Kr_def}, we have 
\begin{equation}\label{eq:main_proj}
    \mathcal{L} \mathcal{K}_r = \Pi_r,
\end{equation}
where $\Pi_r : L^2_\pi \to \text{span}\{\phi_1, \ldots, \phi_r\}$ denotes the orthogonal projector, and $\{\phi_i\}_{i \geq 1}$ are the exact $L^2_\pi$-orthonormal eigenfunctions of $\mathcal{L}$. To quantify the contribution of unretained high-order modes, we define the \emph{spectral tail error}
\begin{equation*}
    \eta_r(f) := \|(I - \Pi_r)f\|_{L^2_\pi} = \left(\sum_{i>r} \langle f, \phi_i \rangle_\pi^2\right)^{1/2}.
\end{equation*}
% For the density fluctuation $f_t$ given in \ref{def:density}, we write $\eta_{r,t} \coloneqq \eta_r(f_t)$.

\begin{assumption}[Regularity]
\label{ass:regularity}
For regularity conditions, we assume that $\rho_t>0 \textit{ a.e.}$ and $\partial_t(\rho_t\log\rho_t)\in L^1_{\pi}$. In addition, we assume that $\nabla f_t\in L^2_{\pi}$ 
and $\mathcal{K}_r f_t\in \mathcal{D}(\mathcal{L})$.
\end{assumption}

\begin{assumption}[Uniform spectral tail bound]
\label{ass:spec_tail_bd}
There exists a constant $\eta_r > 0$ such that
\begin{equation*}
\sup_{t\ge0} \|(I-\Pi_r)f_t\|_{L^2_\pi}
\;\le\; \eta_r.
\end{equation*}
\end{assumption}

\begin{proposition}[Convergence with spectral truncation]
\label{prop:ideal}
Assume that Assumptions~\ref{ass:regularity} and \ref{ass:spec_tail_bd} hold. Let $(\mu_t)_{t\geq 0}$ be the solution to the truncated LAWGD dynamics with exact eigenpairs and initial distribution $\mu_0$:
\begin{equation*}\label{eq:truncated_kpm_LAWGD}
\partial_t\mu_t
= \operatorname{div}\big(\mu_t \nabla \mathcal{K}_r (\mathrm{d}\mu_t/\mathrm{d}\pi)\big),
\end{equation*}
where $\mathcal{K}_r = \sum_{i=1}^r \frac{1}{\lambda_i}\langle \cdot, \phi_i\rangle_\pi\phi_i$ defined in \eqref{eq:Kr_def} satisfies \eqref{eq:main_proj}. Then
\begin{equation}\label{eq:KL_bound}
    \mathrm{KL}(\mu_t \| \pi) \leq \mathrm{KL}(\mu_0 \| \pi) e^{-t} + \eta_r^2(1 - e^{-t}).
\end{equation}
\end{proposition}

\begin{proof}
See Appendix~\ref{app:proof_prop}.
\end{proof}

\begin{remark}
    Here we assumed sufficient decay or integrability at $\infty$ on the domain $\mathbb{R}^d$ so that the “boundary at infinity” term is zero when applying the integration by parts.
\end{remark}

\begin{remark}
Proposition~\ref{prop:ideal} reveals that the idealized KSWGD (using exact but truncated eigenpairs) inherits the scale-free exponential convergence rate $e^{-t}$ from LAWGD, independent of the Poincar\'{e} constant. However, the flow converges to an approximated stationary state with residual error bounded by $\eta_r^2$. This truncation error quantifies the steady-state deviation caused by neglecting spectral components beyond the first $r$ modes, and vanishes as $r\to\infty$ by the completeness of the eigenfunction basis, which recovers exact LAWGD.
\end{remark}

In practice, the eigenpairs $\{(\lambda_i,\phi_i)\}_{i=1}^r$ can be estimated from data using EDMD or kernel-EDMD, which gives approximated eigenpairs $\{(\hat{\lambda}_i,\hat{\phi}_i)\}_{i=1}^r$ and constructs truncated inverse operator $\widehat{\mathcal{K}}_r
= \sum_{i=1}^r \frac{1}{\hat{\lambda}_i}
\langle \cdot , \hat{\phi}_i\rangle_\pi \,\hat{\phi}_i$ as in \eqref{eq:Kr_hat_def}. Therefore, instead of the exact relation $\mathcal{L}\mathcal{K}_r = \Pi_r$, the data-driven operator satisfies a perturbed relation
\begin{equation}\label{eq:main_proj_perturbed}
\mathcal{L} \widehat{\mathcal{K}}_r f = \Pi_r f + \delta_r(f), \quad \forall f \in L^2_\pi,
\end{equation}
where $\delta_r(f)$ captures the Koopman spectral approximation error.

\begin{assumption}[Controlled Operator Approximation Error]
\label{ass:approx_error}
There exists a constant $0<\varepsilon_r \ll 1$ such that for all $t\ge0$,
\begin{equation*}
|\langle f_t,\,\delta_r(f_t)\rangle_{\pi}|
\;\le\;
\varepsilon_r\|f_t\|_{L^2_\pi}^2.
\end{equation*}
\end{assumption}

\begin{remark}
    The constant $\varepsilon_r$ is directly related to the spectral approximation quality of the Koopman operator.
\end{remark}

\begin{theorem}[Error Bound Analysis]
\label{thm:data}
Let Assumptions~\ref{ass:regularity}, \ref{ass:spec_tail_bd} and \ref{ass:approx_error} hold. Let $(\mu_t)_{t\geq 0}$ be the solution to the data-driven KSWGD dynamics with initial distribution $\mu_0$:
\begin{equation*}
    \partial_t\mu_t = \operatorname{div}\big(\mu_t \nabla \widehat{\mathcal{K}}_r (\mathrm{d}\mu_t/\mathrm{d}\pi)\big).
\end{equation*}
Then
\begin{align}\label{eq:data-bound}
\mathrm{KL}(\mu_t\|\pi)
&\;\le\;
e^{-(1-\varepsilon_r)t}\,\mathrm{KL}(\mu_0\|\pi) \notag\\
&\qquad \;+\;
\frac{\eta_r^2}{1-\varepsilon_r}
\Big(1-e^{-(1-\varepsilon_r)t}\Big).
\end{align}
\end{theorem}

\begin{proof}
See Appendix~\ref{app:proof_thm}.
\end{proof}

\begin{remark}
Theorem~\ref{thm:data} quantifies the impact of data-driven spectral approximation. Compared to Proposition~\ref{prop:ideal}, the approximation error $\varepsilon_r$ degrades performance in two ways: (i) it slows the exponential convergence rate from $e^{-t}$ to $e^{-(1-\varepsilon_r)t}$, and (ii) it amplifies the biased error from $\eta_r^2$ to $\eta_r^2/(1-\varepsilon_r)$.
\end{remark}

% -----------------------------------------------------
\subsection{Discrete-Time Analysis via Approximate Gradient Flow}
\label{sec:discrete_agf}

While our analysis in Proposition~\ref{prop:ideal} and Theorem~\ref{thm:data} establishes exponential convergence for the continuous-time dynamics, practical implementation relies on the discrete-time update rule (Algorithm~\ref{alg:kswgd}). To rigorously bridge this gap, we adopt the \textit{Approximate Gradient Flow (AGF)} framework (see \citep{fujisawa2025on} for more details).

The continuous-time dissipation inequality \eqref{eq:data-bound} derived in the proof of Theorem~\ref{thm:data} can be restated in the AGF framework, that is, KSWGD is an $(\alpha, \beta)$-\textit{approximate $\chi^2$-gradient flow} with
\begin{equation*}
    \alpha = 1 - \varepsilon_r \quad \text{and} \quad \beta = \eta_r^2,
\end{equation*}
where $\alpha$ represents the dissipation rate of the update direction and $\beta$ represents the bias floor. 
By adapting the discrete-time analysis of \citep{fujisawa2025on} to our $\chi^2$-geometry, we can directly translate our continuous-time guarantees to the discrete setting.

\begin{corollary}[Discrete-Time Linear Convergence]
\label{cor:discrete_convergence}
Consider the discrete KSWGD update $\mu_{t+1} = (\mathrm{I} - h \widehat{v}_t)\#\mu_t$ with step size $h > 0$ where $\widehat{v}_t=-\nabla \widehat{\mathcal{K}}_r (\mathrm{d}\mu_t/\mathrm{d}\pi)$ \citep{pmlr-v162-salim22a}. Under the conditions of Theorem~\ref{thm:data} and assuming $\mu_t$ is sufficiently close to $\pi$ (i.e., $\chi^2(\mu_t\|\pi) \ge \mathrm{KL}(\mu_t\|\pi)$), the discrete iterations satisfy:
\begin{equation}\label{eq:disc_iter}
    \mathrm{KL}(\mu_{t+1} \| \pi) \le (1 - \alpha h) \mathrm{KL}(\mu_t \| \pi) + h\beta + \mathcal{O}(h^2).
\end{equation}
Iterating this bound gives linear geometric convergence to the noise floor $\beta/\alpha$:
\begin{equation*}
    \mathrm{KL}(\mu_T \| \pi) \le (1 - \alpha h)^T \mathrm{KL}(\mu_0 \| \pi) + \frac{\beta}{\alpha} + \mathcal{O}(h).
\end{equation*}
\end{corollary}

%%%%%%%%%%%%%%%%
\begin{proof}
See Appendix~\ref{app:proof_cor}.
\end{proof}
\begin{remark}
The boundedness of $\mathcal{O}(h^2)$ in \eqref{eq:disc_iter} is controllable; refer to \citep{fujisawa2025on,pmlr-v162-salim22a} for more details.
\end{remark}

Corollary~\ref{cor:discrete_convergence} bridges the gap between our continuous-time theory and the practical Algorithm~\ref{alg:kswgd}. Specifically, by identifying the physical time $t = Th$, the discrete geometric decay explicitly recovers the exponential rate derived in Section~\ref{sec:main_results}, i.e., $(1 - \alpha h)^{t/h} \approx e^{-\alpha t}$ as $h \to 0$. Furthermore, the discrete analysis reveals an unavoidable bias $\mathcal{O}(h)$ added to the noise floor which provides a theoretical guide for the efficiency-accuracy trade-off: a smaller step size $h$ reduces bias but requires more iterations. This detail is not obvious in the continuous-time analysis. Finally, we also note that while standard SVGD requires a strong eigenvalue lower bound \citep[Assumption 6]{fujisawa2025on} to guarantee a bounded approximation error $\epsilon_t$ (a condition often violated by RBF kernels), KSWGD automatically guarantees a bounded error (i.e., constant dissipation rate $\alpha = 1 - \varepsilon_r$) by construction. This is because the rank-$r$ truncation restricts dynamics to the spectral subspace where eigenvalues are strictly lower-bounded, that structurally enables the convergence without external eigenvalue assumptions.

\begin{remark}
The AGF framework clarifies the theoretical advantage of KSWGD over SVGD. As shown by \citep{fujisawa2025on}, standard SVGD suffers from a \textit{decaying} dissipation rate (i.e., $\alpha_t \to 0$), caused by the rapid eigenvalue decay of RBF kernels in RKHS; this inevitably leads to sub-linear convergence. In contrast, KSWGD leverages spectral preconditioning to maintain a \textit{constant} dissipation rate ($\alpha = 1 - \varepsilon_r > 0$) throughout the optimization. This structural difference of preventing the vanishing gradient rate is the fundamental reason why KSWGD achieves linear convergence while SVGD does not in discrete iteration.
\end{remark}

%%%%%%%% Section 5 %%%%%%%%
\section{Feynman--Kac Interpretation of KSWGD}\label{sec:feynman-kac}
The Koopman-based formulation of KSWGD admits a coherent interpretation within the Feynman--Kac framework, which places the proposed method in a unified probabilistic and operator-theoretic setting. Specifically, KSWGD corresponds to the zero-potential case ($U \equiv 0$) of the associated Kolmogorov backward equation, characterizing unconditional expectations over stochastic trajectories. This perspective clarifies that KSWGD targets unconditional sampling and distributional prediction, without implicit path conditioning or importance reweighting. Although not essential for the algorithmic construction or convergence analysis, the Feynman--Kac viewpoint provides a concise summary of the method’s probabilistic interpretation and highlights natural extensions to conditional sampling problems. Further details are given in Appendix~\ref{app:feynman-kac}.

%%%%%%%% Section 6 %%%%%%%%
\section{Experiments}\label{sec:experiment}
This section evaluates KSWGD empirically across three representative regimes. 
Section~\ref{sec:torus} studies sampling on a compact manifold (a 2D torus) to test geometric adaptivity, Section~\ref{sec:quadruple_well} considers a metastable multi-well system to test multi-modal sampling and barrier-crossing dynamics, and Section~\ref{sec:allen_cahn} targets a high-dimensional SPDE (stochastic Allen--Cahn) to assess scalability. 
Across these settings we compare our method with DMPS and other standard neural generative baselines.

%%%%%%%% Section 6.1 %%%%%%%%
\subsection{Torus $\mathbb{T}^2$}\label{sec:torus}
We apply the Kernel-EDMD combined with the KSWGD algorithm to the 2D torus manifold $\mathbb{T}^2$ embedded in $\mathbb{R}^3$. The dynamics on the torus are modeled by Riemannian Brownian motion, which admits the canonical volume measure as its invariant distribution (see Eq.~\eqref{eq:torus} for the explicit parametrization of $\mathbb{T}^2$ and Eq.~\eqref{eq:stra_sde_torus} for the SDE formulation in Appendix~\ref{app:t_2}).

The Koopman operator is approximated via Kernel-EDMD from trajectory pairs sampled on the torus. For the KSWGD algorithm, particles are initialized in a localized region (in red) near the top of the torus tube. The gradient computation employs a data-driven Riemannian metric structure, where the tangent space at each particle is estimated via local PCA using nearest neighbors in the training data. This approach enables the KSWGD method to naturally adapt to the underlying manifold geometry without requiring explicit knowledge of the manifold parametrization.

Results in Figure~\ref{fig:torus_kswgd} demonstrate successful diffusion of particles from the initial localized configuration (red) toward the uniform target distribution (blue), while preserving the manifold geometry with only a slight deviation from the torus surface throughout the evolution (magenta). Figure~\ref{fig:torus_kswgd2} shows a $3\times3$ overlapping scatter matrix comparing the target distribution (blue) and final particle positions (red) in 3D Cartesian coordinates (x, y, z): the diagonal panels show overlapping histograms of each coordinate's marginal distribution, while the off-diagonal panels display overlapping 2D scatter plots of pairwise coordinates, which displays a direct visual assessment of how well the transported particles match the target distribution in all coordinate projections. A comparison with DMPS method is shown in Figure~\ref{fig:torus_dmps} in Appendix.
\begin{figure*}[!ht]
    \centering
    \includegraphics[width=0.75\linewidth]{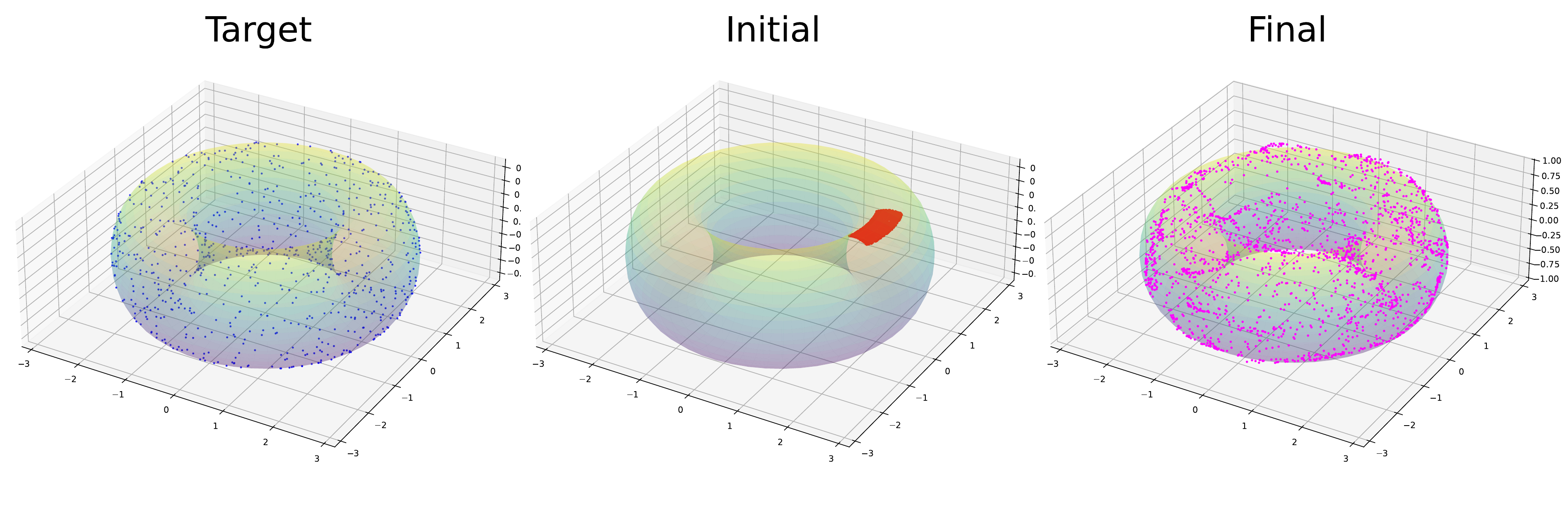}
    \caption{Three-dimensional visualization of particle evolution on the torus. \textbf{(Left)} Target distribution. \textbf{(Center)} Initial configuration. \textbf{(Right)} Final configuration after KSWGD iterations.}
    \label{fig:torus_kswgd}
\end{figure*}

\begin{figure*}[!ht]
    \centering
    \includegraphics[width=0.70\linewidth]{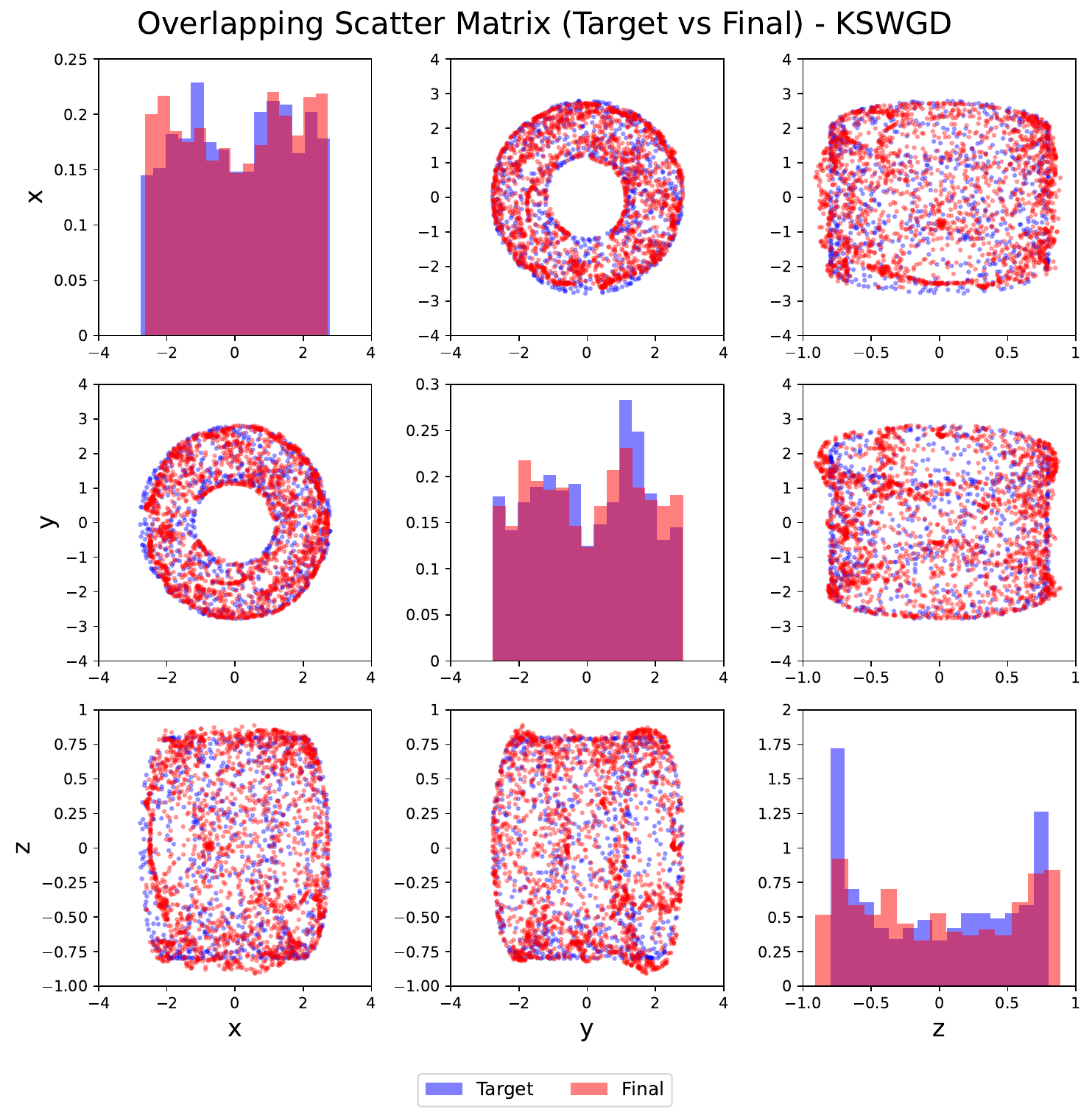}
    \caption{Scatter matrix comparing target (blue) vs. final particle distribution (red) in 3D Cartesian space.}
\label{fig:torus_kswgd2}
\end{figure*}

%%%%%%%% Section 6.2 %%%%%%%%
\subsection{Quadruple Well Potential System}\label{sec:quadruple_well}
We consider a metastable two-dimensional quadruple-well potential $V(x,y) = (x^2-1)^2 + (y^2-1)^2$ characterized by four local minima at $(\pm 1, \pm 1)$. The objective is to sample from the Boltzmann distribution $\pi(x) \propto \exp(-V(x))$ of the associated overdamped Langevin dynamics without explicit knowledge of $V$. The training dataset consists of consecutive time-series pairs generated from the stationary distribution, and the Koopman operator is approximated using SDMD \citep{10.1063/5.0283640} with a neural network-based dictionary (see Appendix~\ref{app:quadruple} for detailed parameter settings).

As illustrated in Figure~\ref{fig:heat_map}, we compare KSWGD with DMPS alongside the Kernel Density Estimation (KDE) of the training data. KSWGD successfully recovers nearly equal density across all four wells, despite the empirical distribution of the training samples being unevenly distributed; in contrast, DMPS produces only a blurred region that conveys limited structural information. In Appendix~\ref{app:mode_collapse}, we further investigate the mode collapse phenomenon and verify that smaller kernel bandwidth values lead to more severe mode collapse \citep{ba2021understanding,liu2016stein}, which is consistent with the analysis in \citep[Section~3.2]{liu2016stein}. Additional comparisons with DMPS are provided in Appendix~\ref{app:quadruple}.

\begin{remark}
    Note that though we use a neural network to train the Koopman matrix approximation, the gradient descent part of the generative process remains purely numerical and training-free.
\end{remark}
\begin{figure}[ht]
    \centering
    \includegraphics[width=0.95\linewidth]{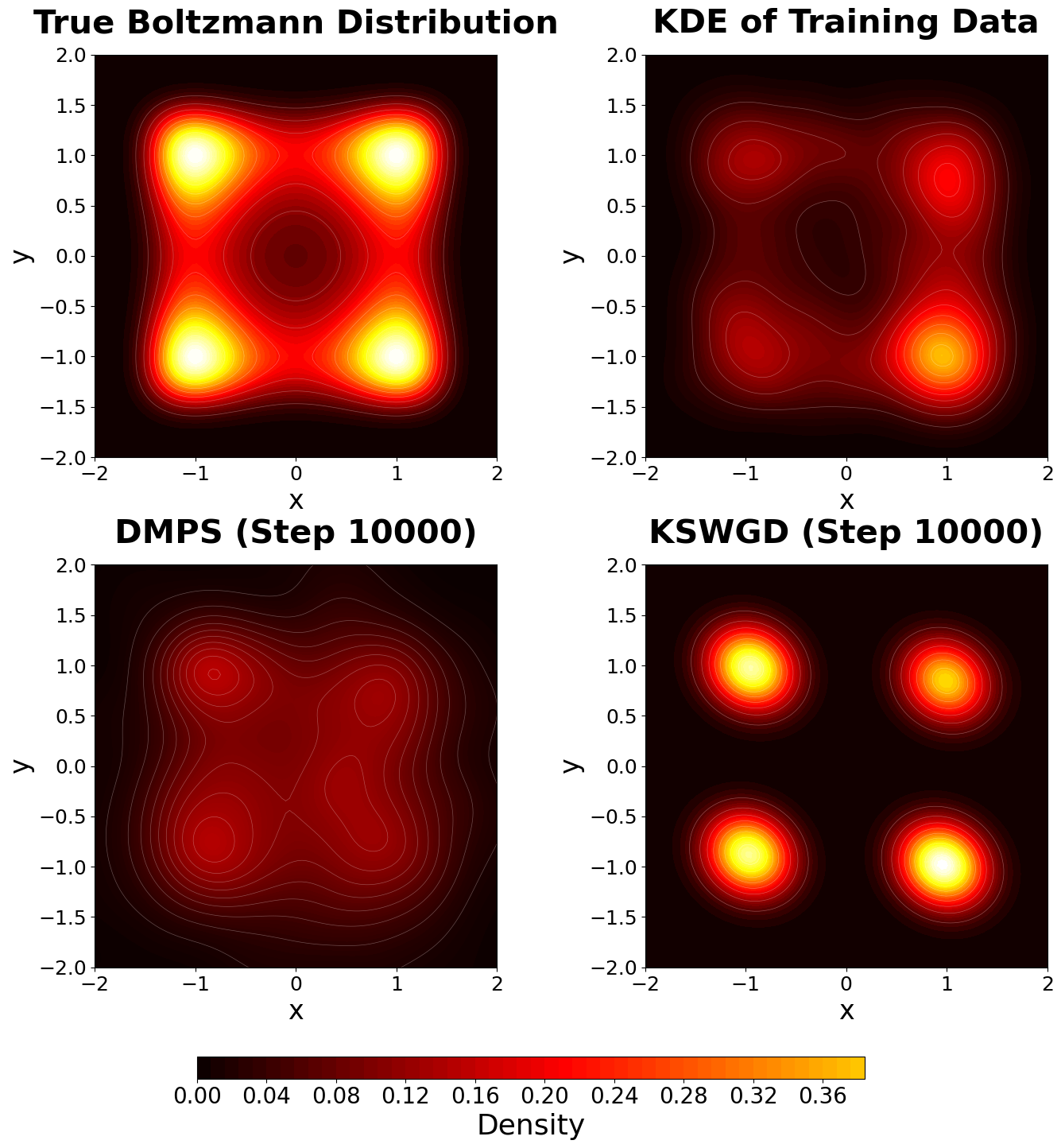}
    \caption{Heat map of DMPS, KSWGD at step 10,000 and KDE of the training data.}
    \label{fig:heat_map}
\end{figure}

%%%%%%%% Section 6.3 %%%%%%%%
% \begin{figure*}[!ht]
%     \centering
%     \begin{subfigure}{0.95\linewidth}
%         \centering
%         \includegraphics[width=\linewidth]{images/allen_cahn/ac_14_2.png}
%         \caption{Allen--Cahn equation example using KSWGD combined with EDMD (polynomial dictionary).}
%         \label{fig:ac_kswgd1}
%     \end{subfigure}

%     \vspace{0.3cm}

%     \begin{subfigure}{0.95\linewidth}
%         \centering
%         \includegraphics[width=\linewidth]{images/allen_cahn/ac_12_2.png}
%         \caption{Comparison of KSWGD with DDPM, VAE, normalizing flows, and GAN.}
%         \label{fig:ac_comparison1}
%     \end{subfigure}

%     \caption{Learning the stochastic Allen-Cahn dynamics via KSWGD to approximate phase-field evolution.}
%     \label{fig:ac_main}
% \end{figure*}
\subsection{Stochastic Allen-Cahn Equation}\label{sec:allen_cahn}
We evaluate the predictive capability of KSWGD on the stochastic Allen-Cahn equation \citep{allen1972ground,steinbach2009phase}, which models phase separation kinetics driven by thermal fluctuations and serves as a benchmark for testing generative methods on SPDEs. The system is governed by:
\begin{equation}\label{eq:ac}
    \mathrm{d} u = \left[ D\nabla^2 u - \epsilon^{-2} u(u^2-1) \right] \mathrm{d}t + \sigma \mathrm{d}W_t,
\end{equation}
where $u(x,t)$ is the phase field, $D$ is the diffusion coefficient, and $W_t$ represents the standard Wiener process multiplied by noise intensity $\sigma$. The parameter $\epsilon$ defines the width of the interfacial transition layers. A critical property of this system is its stiffness: as $\epsilon \to 0$, the reaction term $\epsilon^{-2} u(u^2-1)$ becomes large, forcing rapid transitions towards the stable equilibria $u=\pm 1$. This necessitates very fine temporal resolution to maintain numerical stability, making data-driven modeling particularly challenging due to the high-dimensional spatiotemporal correlations. In this experiment, we validate KSWGD's capability for learning these spatiotemporal distributions. We simulate the equation on a high-resolution grid and compress each snapshot to a low-dimensional latent space via a fully-connected autoencoder. Using EDMD with a polynomial dictionary, we construct the Koopman matrix $K$ in latent space from paired data only at $t_0$ and $t_1$. During the KSWGD sampling phase, particles are initialized in the latent space and evolved to generate samples matching the initial distribution. Subsequently, we apply the Koopman matrix $K$ successively to predict future distributions; more specifically, $K^2$ predicts at $t=0.0002$, and $K^5$ predicts at $t=0.0005$. Note that the snapshots at $t_2$ and $t_3$ are not used in training and serve solely to validate the Koopman operator's predictive capability (see Appendix~\ref{app:stochastic_ac_eqn} for detailed parameter settings and Appendix~\ref{app:latent_prediction} for latent-space prediction analysis).

Figure~\ref{fig:ac_kswgd1} shows that the Koopman operator only ``sees'' one spike at $t_0$ and $t_1$, but still captures the dynamics and predicts the two spikes at future time $t_2$ and $t_3$, which is highly consistent with the ground truth. This reflects the operator's ability to capture the essential distributional dynamics underlying the physical trajectory evolution. In addition, we also compare KSWGD against several baseline generative methods, including Diffusion Modeling (DDPM) \citep{ho2020denoising,nichol2021improved}, VAE \citep{kingma2013auto}, Normalizing Flows \citep{dinh2017realnvp}, and GAN \citep{gulrajani2017improved}. All methods operate in the same shared latent space for fair comparison. As shown in Figure~\ref{fig:ac_comparison1}, KSWGD demonstrates competitive performance. This comparison is not meant to suggest that image-oriented generative methods are inadequate, on the contrary, they are highly effective in their intended domains, but rather to illustrate that Koopman-based approaches can be particularly well-suited for problems with inherent temporal dynamics. Additional experiments with different interface width parameter $\epsilon$ are provided in Figure~\ref{fig:ac_21} and Figure~\ref{fig:ac_comparison2} of Appendix~\ref{app:stochastic_ac_eqn}.

%%%%%%%% This is 2x2 layout figure %%%%%%%% 
\begin{figure}[!htb]
    \centering
    % --- Subfigure (a) ---        
    \begin{subfigure}{0.95\columnwidth}
        \centering
        % Using 0.95 of column width to prevent overflow and maintain margins
        \includegraphics[width=0.99\linewidth]{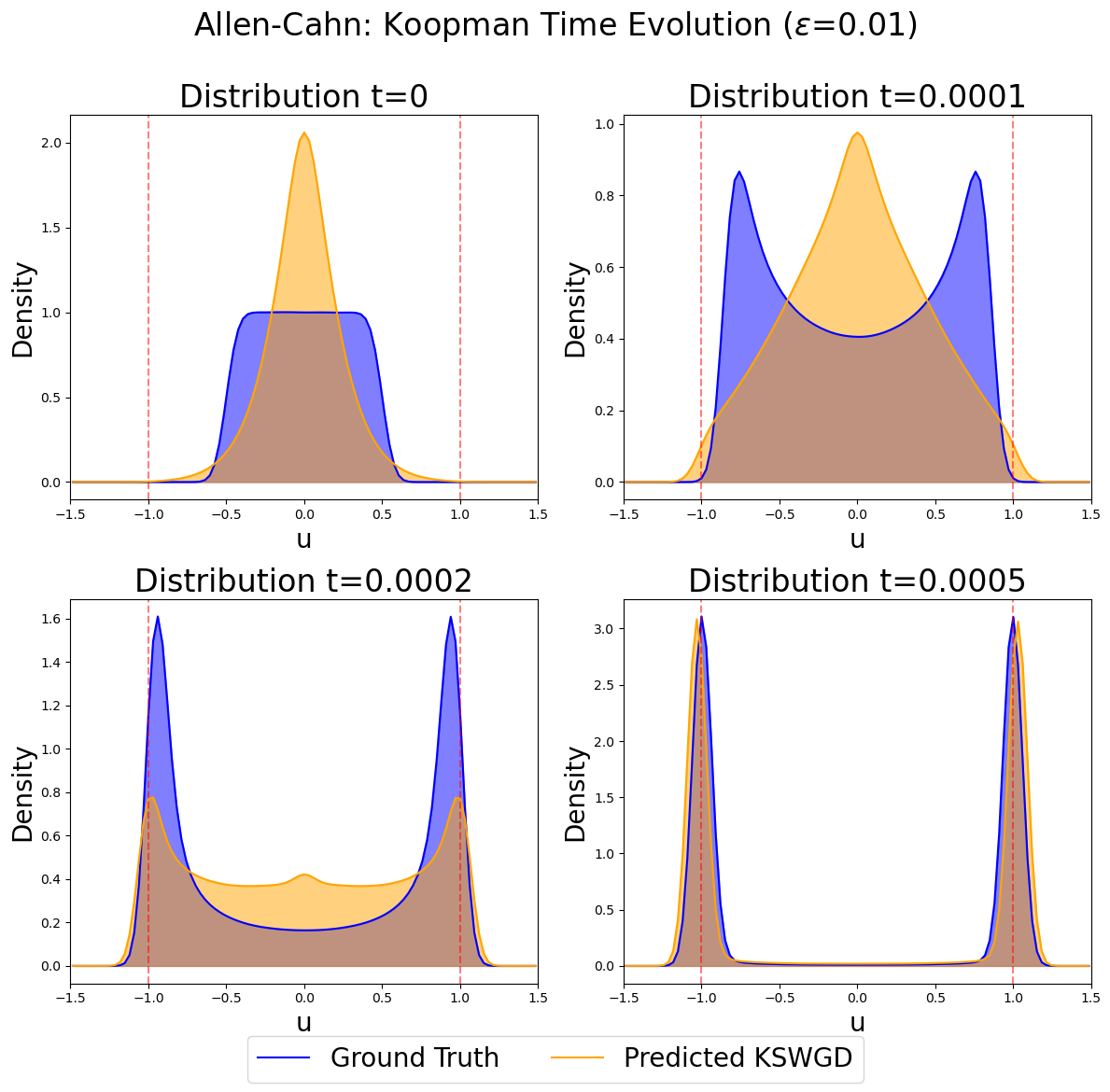}
        \caption{Allen-Cahn equation example using KSWGD with EDMD (Polynomial).}
        \label{fig:ac_kswgd1}
    \end{subfigure}

    \vspace{0.1cm} % Adjust the vertical gap between the two subfigures

    % --- Subfigure (b) ---        
    \begin{subfigure}{0.95\columnwidth}
        \centering
        \includegraphics[width=0.99\linewidth]{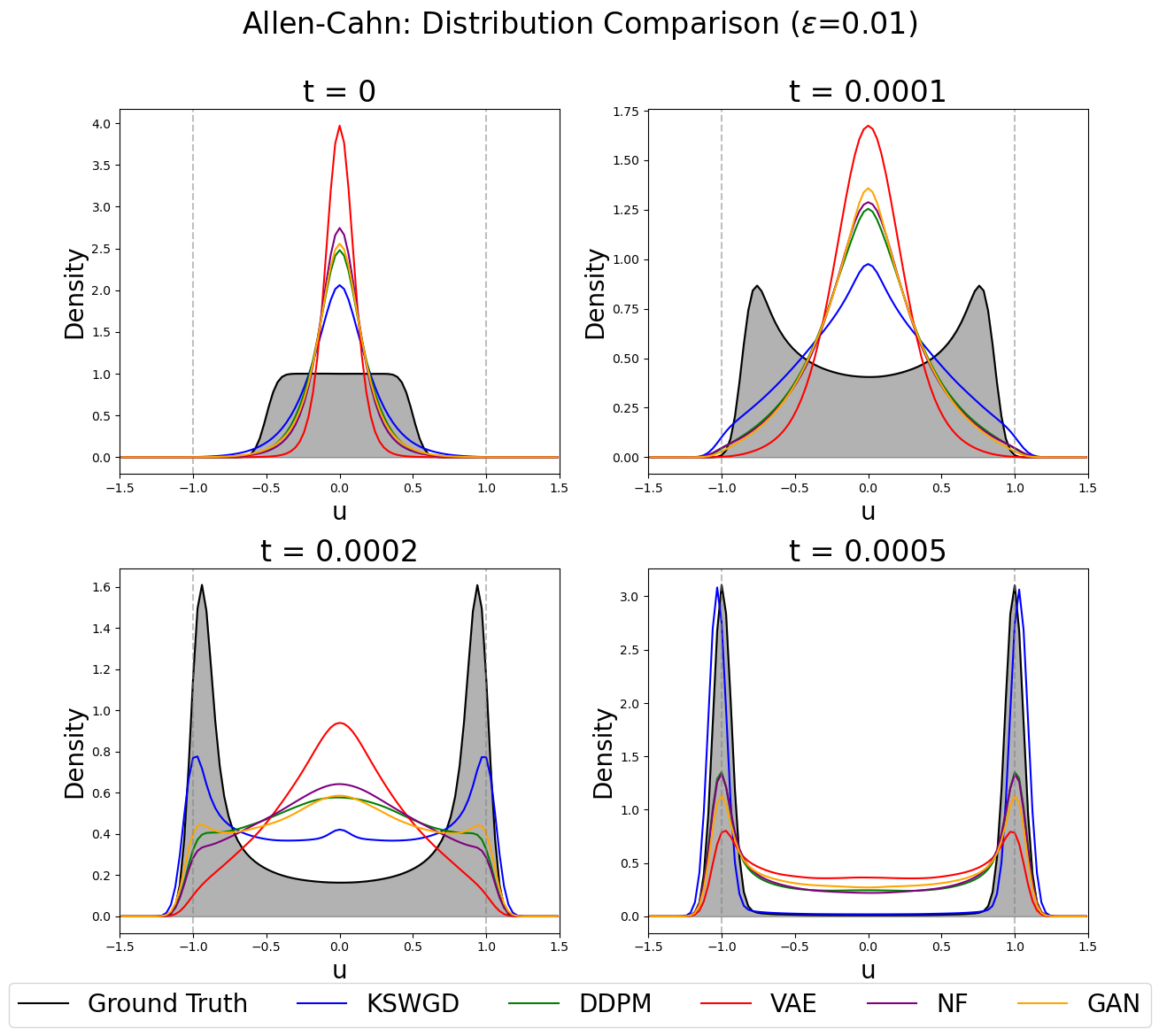}
        \caption{Comparison: KSWGD vs DDPM, VAE, Normalizing Flows and GAN.}
        \label{fig:ac_comparison1}
    \end{subfigure}    
    \caption{Learning the stochastic Allen-Cahn dynamics via KSWGD to approximate phase-field evolution.}
    \label{fig:ac_main}
\end{figure}

%%%%%%%% Section 7 %%%%%%%%
\section{Conclusion}
In this work, we introduced KSWGD, a training-free, particle-based generative modeling framework that connects Koopman operator theory with Wasserstein gradient flows. Exploiting the link between the Koopman generator and the Langevin generator, KSWGD constructs a spectral kernel in a data-driven manner without access to the target potential. Our analysis establishes scale-free exponential convergence, with explicit error bounds determined by the truncation rank and the quality of the Koopman approximation. Compared with SVGD, spectral preconditioning maintains a constant dissipation rate, mitigating the vanishing-gradient behavior that typically leads to sublinear convergence. We also provide a Feynman--Kac interpretation that places KSWGD in a broader operator-theoretic setting. Extending the framework to conditional sampling with non-zero killing rates is an interesting direction for future work. Experiments on compact manifolds, metastable systems, and high-dimensional SPDEs demonstrate the effectiveness of our approach.

% \section*{Acknowledgements}
% Authors are grateful to Ben Gao, Hanru Bai, Zhuo Sun, Sinho Chewi, Yunan Yang, Micah Warren, Maria Oprea, Zhongwei Shen, Yao Li, Zhicheng Jiang, Qiao Sun, Vladimir Kostic, for all kinds of valuable discussion and comments. Y.X. and I.I. acknowledge support from JST CREST Grant Number JPMJCR24Q1 including AIP challenge program, Japan. 

% \textbf{Do not} include acknowledgements in the initial version of the paper

% If a paper is accepted, the final camera-ready version can (and usually should)
% include acknowledgements.  Such acknowledgements should be placed at the end of
% the section, in an unnumbered section that does not count towards the paper
% page limit. Typically, this will include thanks to reviewers who gave useful
% comments, to colleagues who contributed to the ideas, and to funding agencies
% and corporate sponsors that provided financial support.

\section*{Impact Statement}
This paper presents work whose goal is to advance the field of Machine
Learning. There are many potential societal consequences of our work, none which we feel must be specifically highlighted here.

\bibliography{example_paper}
\bibliographystyle{icml2026}

%%%%%%%%%%%%%%%%%%%%%%%%%%%%%%%%%%%%%%%%%%%%%%%%%%%%%%%%%%%%%%%%%%%%%%%%%%%%%%%
%%%%%%%%%%%%%%%%%%%%%%%%%%%%%%%%%%%%%%%%%%%%%%%%%%%%%%%%%%%%%%%%%%%%%%%%%%%%%%%
% APPENDIX
%%%%%%%%%%%%%%%%%%%%%%%%%%%%%%%%%%%%%%%%%%%%%%%%%%%%%%%%%%%%%%%%%%%%%%%%%%%%%%%
%%%%%%%%%%%%%%%%%%%%%%%%%%%%%%%%%%%%%%%%%%%%%%%%%%%%%%%%%%%%%%%%%%%%%%%%%%%%%%%
\newpage
\appendix
\onecolumn
\input{supplementary}

% You can have as much text here as you want. The main body must be at most $8$
% pages long. For the final version, one more page can be added. If you want, you
% can use an appendix like this one.

% The $\mathtt{\backslash onecolumn}$ command above can be kept in place if you
% prefer a one-column appendix, or can be removed if you prefer a two-column
% appendix.  Apart from this possible change, the style (font size, spacing,
% margins, page numbering, etc.) should be kept the same as the main body.
%%%%%%%%%%%%%%%%%%%%%%%%%%%%%%%%%%%%%%%%%%%%%%%%%%%%%%%%%%%%%%%%%%%%%%%%%%%%%%%
%%%%%%%%%%%%%%%%%%%%%%%%%%%%%%%%%%%%%%%%%%%%%%%%%%%%%%%%%%%%%%%%%%%%%%%%%%%%%%%

\end{document}

%% file: supplementary.tex
\section{Experiments}\label{sec:appendix}
Here we are going to show extra comparison tests between KSWGD and other baseline methods such as DMPS, etc. in various systems that are discussed in the main content. We also provide extra necessary background and clarification along with the experiment results.

\subsection{Source Code}
For reproducibility, the source code will be available at this \href{https://github.com/TalkingDoll/kswgd}{link}.

\subsection{Quadruple Potential Well System}\label{app:quadruple}

\paragraph{Training data generation.} We generate a trajectory from the stationary distribution of the Quadruple well potential system as proposed in Section~\ref{sec:quadruple_well} via MCMC sampling and evolve each state forward using an Euler–Maruyama discretization with time step $\Delta t = 0.01$. This produces 10,000 consecutive time-series pairs $(X_t, X_{t+\Delta t})$ for Koopman operator approximation.

\paragraph{KSWGD parameters.} We initialize 2,000 particles from the standard 2-dimensional normal distribution and apply KSWGD for $T=10,000$ iterations with step size $h=0.02$. The kernel bandwidth is set to $\text{med}^2/(2\log N)$ following \citep{liu2016stein}.

\paragraph{Comparison with DMPS.} When the kernel bandwidth is much smaller, DMPS converges significantly more slowly compared to KSWGD, with many particles remaining outside the wells at iteration 10,000, as shown in Figure~\ref{fig:quadruple_kswgd_dmps}. Both methods become stable after 10,000 steps, as shown in Figure~\ref{fig:kl}, which plots the average $\text{Movement Rate}=\frac{1}{N}\sum_{i=1}^N \lVert x_i^{(t+1)}-x_i^{(t)}\rVert$ as a direct indicator of dynamical stability. A movement rate near zero indicates that particles have effectively stopped moving. 

\subsection{Mode Collapse}\label{app:mode_collapse}
As pointed out in \citep[Section~3.2]{liu2016stein}, deterministic particle-based sampling methods face inherent variance collapse issues. As in most variational inference methods, the particles over-concentrate in high-density regions, which is analogous in the Quadruple well potential system example here; specifically, the particles collapse toward potential well minima, which gives a characteristic U-shaped KL divergence trajectory, as shown in Figure~\ref{fig:kl2} and \ref{fig:quadruple_kswgd2}. As illustrated in Figure~\ref{fig:quadruple_kswgd2}, the red dots denote the initial particle from uniform distribution, while the purple circles indicate their final positions after 3,000 steps. It can be observed that each particle is successfully transported to and stops at the basin of attraction (i.e., the well) nearest to its initialization. Interestingly, we also see that the "bottom" of the wells are not in the center, which indicates the existence of bias here.

\subsection{CelebA-HQ}\label{app:celeba}
In this section, we provide a preliminary exploration of applying KSWGD to image generation on the CelebA-HQ dataset. It is important to note that our KSWGD framework built upon Koopman operator approximation is inherently designed for dynamical systems where temporal evolution plays a central role. Consequently, the method is naturally suited for time-series data and sequential generative modeling tasks. Nevertheless, we investigate its applicability to static image generation as a proof-of-concept experiment, similar as the $S^1$ example in Appendix~\ref{app:s_1}.

To handle the high dimensionality, we leverage the latent space of a pre-trained VAE from the LDM framework (CompVis/ldm-celebahq-256 \citep{NEURIPS2022_62868cc2,rombach2021highresolution}). Specifically, we encode 30,000 real images into the VAE latent space and further reduce the dimensionality to 8 via an MLP AutoEncoder to establish the target distribution. The generative process involves evolving 16 particles, initialized from a standard normal distribution, for 500 iterations with a step size of $h=0.001$. The transported particles are then projected back and reconstructed into images via the decoder.

As shown in Figure~\ref{fig:celeba_with_table}, while KSWGD can generate recognizable facial images without extensive hyperparameter tuning, the FID scores (evaluated on 1,000 generated samples due to the limitation of computational resource) indicate a notable performance gap compared to Latent Diffusion Models with 200 diffusion steps. This is expected, as diffusion-based methods are specifically architected for image synthesis with sophisticated noise schedules and network designs. Our results suggest that adapting Koopman-based approaches to image generation remains a promising but challenging direction. Potential improvements may include incorporating temporal structure into the generation process or developing hybrid architectures that combine the dynamical systems perspective with modern generative frameworks. We leave these explorations for future work.

%%%%%%%% %%%%%%%%
\subsection{Brownian Motion on Compact (Riemannian) Manifolds}
This section provides a comprehensive explanation of the mathematical framework underlying the KSWGD algorithm on compact (Riemannian) manifolds, with specific focus on the surface of 2D-torus $\mathbb{T}^2$ and 1D-shere $S^1$ implementation.

\subsubsection{Torus $\mathbb{T}^2$}\label{app:t_2}
The torus $\mathbb{T}^2$ in Section~\ref{sec:torus} is parametrized by angular coordinates $(\theta, \phi) \in [0, 2\pi)^2$ as
\begin{equation}\label{eq:torus}
    x = (R + r\cos\phi)\cos\theta, \quad y = (R + r\cos\phi)\sin\theta, \quad z = r\sin\phi,
\end{equation}
where $R = 2.0$ is the major radius (distance from the origin to the tube center) and $r = 0.8$ is the minor radius (tube radius).

\paragraph{Dynamics on $\mathbb{T}^2$.} We model Brownian motion on the torus using the Stratonovich SDE:
\begin{equation}\label{eq:stra_sde_torus}
    \mathrm{d}\theta = \frac{\sqrt{2}}{R + r\cos\phi} \circ \mathrm{d}W_t^1, \quad \mathrm{d}\phi = \frac{\sqrt{2}}{r} \circ \mathrm{d}W_t^2,
\end{equation}
which corresponds to Riemannian Brownian motion with invariant measure $\mathrm{d}\mu = r(R+r \cos \phi) \,\mathrm{d}\theta \,\mathrm{d}\phi$ \cite{hsu2002stochastic}. For numerical implementation, we use the equivalent It\^{o} form with the Wong-Zakai correction term (see Appendix~\ref{app:stra_sde_manifold} and ~\ref{app:ito_stra_correction}).

\paragraph{Koopman operator approximation.} Training data pairs $\{(\bm{x}^i, \hat{\bm{x}}^i_{\Delta t})\}_{i=1}^N$ are constructed by sampling $N = 800$ initial points uniformly on $\mathbb{T}^2$ and advancing them by $\Delta t = 0.05$. The kernel matrices $K_{xx}$ and $K_{xy}$ are computed using the RBF kernel $k(x,y) = \exp(-\|x-y\|^2/(2\varepsilon))$ with bandwidth $\varepsilon$ determined by the median heuristic. The Koopman operator is approximated as $K = K_{xy}(K_{xx} + \gamma I)^{-1}$ with ridge regularization $\gamma = 10^{-6}$.

\paragraph{KSWGD parameters.} We initialize $m = 2000$ particles in a localized region near $\theta \approx 0$ and $\phi \approx \pi/2$ (the top of the torus tube) with angular spreads of $0.3$ radians in both directions. The tangent space at each particle is estimated via local PCA using $k = 30$ nearest neighbors. The algorithm is run for $5000$ iterations with step size $h = 0.005$.

\subsubsection{Stratonovich SDE on Manifolds}\label{app:stra_sde_manifold}
Brownian motion on a Riemannian manifold $(\mathcal M,g)$, where $\mathcal M$ is a smooth $n$-dimensional manifold endowed with a Riemannian metric $g=(g_{ij})$, can be represented in local coordinates $(q^1,\ldots,q^n)$ by a \emph{Stratonovich} stochastic differential equation \cite{hsu2002stochastic} of the form
\[
dq^i=\sigma^i_j(q)\circ dW_t^j,\qquad i=1,\ldots,n,
\]
where $W_t^j$ are independent standard Brownian motions and $\circ$ denotes Stratonovich integration.
The diffusion coefficients $\sigma^i_j(q)$ are chosen so that the induced diffusion tensor coincides with the inverse metric, namely
\[
\sum_k \sigma^i_k(q)\,\sigma^j_k(q)=g^{ij}(q).
\]
With this choice, interpreting the columns of $\sigma$ as a local orthonormal frame, the resulting process has generator $\tfrac12\Delta_g$ under the standard convention. Moreover, the associated semigroup is symmetric with respect to the Riemannian volume measure; in particular, on a compact manifold without boundary, the normalized volume measure is invariant. Notice that Stratonovich SDEs are coordinate-invariant and naturally preserve manifold constraints, so no additional correction drift is required to account for the metric structure.

\subsubsection{It\^{o}-Stratonovich Conversion (Wong-Zakai Correction)}\label{app:ito_stra_correction}
For numerical simulation, we often convert Stratonovich SDEs to It\^{o} form. The conversion introduces an additional drift term (It\^{o} correction):
$$
dq^i = \underbrace{b^i(q) \, dt}_{\text{It\^{o} correction}} + \sigma^i_j(q) \, dW^j_t
$$
where the drift term is:
$$
b^i = \frac{1}{2} \sum_{j,k} \sigma^k_j \frac{\partial \sigma^i_j}{\partial q^k}.
$$
The It\^{o} correction arises because that the diffusion coefficients $\sigma^i_j$ depend on position $q$. In Stratonovich calculus, the noise "sees" the gradient of $\sigma$, but this also creates a bias toward regions where diffusion is stronger. Notice that the It\^{o} correction is related to the Christoffel symbols \cite{do1992riemannian} (connection coefficients) of the manifold which ensures that the process respects the curved geometry.

\subsubsection{Torus-Specific SDE}

Consider the two-dimensional torus $\mathbb{T}^2$ embedded in $\mathbb{R}^3$ with parametrization
\[
X(\theta, \phi)
= \bigl((R + r\cos\phi)\cos\theta,\,
        (R + r\cos\phi)\sin\theta,\,
        r\sin\phi\bigr),
\quad (\theta,\phi)\in[0,2\pi)^2.
\]
The induced Riemannian metric in angular coordinates is
\[
g = \mathrm{diag}\bigl((R + r\cos\phi)^2,\; r^2\bigr).
\]

The Brownian motion on $\mathbb{T}^2$ associated with the uniform invariant measure is given in Stratonovich form by
\[
\begin{aligned}
d\theta &= \frac{\sqrt{2}}{R + r\cos\phi} \circ dW^1_t,\\
d\phi &= \frac{\sqrt{2}}{r} \circ dW^2_t.
\end{aligned}
\]

Converting to It\^{o} form would create an additional drift term in the $\theta$-equation; more specifically,
$$
d\theta = \frac{\sin\phi}{r(R + r\cos\phi)^2}\, dt
          + \frac{\sqrt{2}}{R + r\cos\phi}\, dW^1_t,
$$
which arises from the It\^{o}--Stratonovich correction induced by the $\phi$-dependence of the metric coefficient.

\subsubsection{Manifolds and Tangent Spaces}
A smooth manifold $\mathcal{M}$ of dimension $n$ embedded in $\mathbb{R}^d$ (where $d \geq n$) is a space that locally looks like $\mathbb{R}^n$. At each point $x \in \mathcal{M}$, the tangent space $T_x \mathcal{M}$ is the $n$-dimensional vector space of all possible velocity vectors of curves passing through $x$:
$$
T_x \mathcal{M} = \left\{ v \in \mathbb{R}^d : v = \frac{d}{dt}\bigg|_{t=0} \gamma(t), \, \gamma: [0,1] \to \mathcal{M}, \, \gamma(0) = x \right\}.
$$
Notice that $T_x \mathcal{M}$ is a linear subspace of $\mathbb{R}^d$ of dimension $n$ and the normal space $N_x \mathcal{M} = (T_x \mathcal{M})^\perp$ is the orthogonal complement. So, any vector $v \in \mathbb{R}^d$ can be decomposed as $v = v_{\text{tangent}} \oplus v_{\text{normal}}$. For example, consider the torus $\mathbb{T}^2 \subset \mathbb{R}^3$ at each point $x \in \mathbb{T}^2$ in our case, the tangent space $T_x \mathbb{T}^2$ is two-dimensional, while the normal space $N_x \mathbb{T}^2$ is one-dimensional. In particular, the tangent space is spanned by the coordinate vector fields
$\partial X / \partial \theta$ and $\partial X / \partial \phi$ associated with
the angular parametrization of the torus.

\subsubsection{Riemannian Metric and Gradient}
A Riemannian metric $g$ assigns an inner product to each tangent space
$T_x \mathcal{M}$. This structure allows us to define gradients of functions directly on the manifold. In particular, it allows us to define gradients of functions in a coordinate-free manner. For a smooth function $f : \mathcal{M} \to \mathbb{R}$, the Riemannian gradient
$\nabla_{\mathcal{M}} f(x)$ is defined implicitly by the relation
\[
g_x\bigl(\nabla_{\mathcal{M}} f(x), v\bigr) = D_vf(x),
\qquad \forall v \in T_x \mathcal{M},
\]
where $D_vf(x)$ denotes the directional derivative of $f$ at $x$ along $v$. When $\mathcal{M} \subset \mathbb{R}^d$ is an embedded submanifold equipped with
the metric induced from the ambient Euclidean space, the Riemannian gradient
admits a simple expression:
\[
\nabla_{\mathcal{M}} f(x)
= \Pi_{T_x \mathcal{M}}\bigl(\nabla f(x)\bigr),
\]
that is, it is obtained by orthogonally projecting the ambient Euclidean gradient
onto the tangent space.

\subsubsection{Computing Tangent Space}

If the manifold has a known parametrization $X: U \subset \mathbb{R}^n \to M \subset \mathbb{R}^d$, then:
$$
T_x \mathcal{M} = \text{span}\left\{ \frac{\partial X}{\partial q^1}, \ldots, \frac{\partial X}{\partial q^n} \right\}.
$$
For example, in the torus $\mathbb{T}^2$, we have:
$$
\frac{\partial X}{\partial \theta} = \begin{pmatrix}
-(R + r\cos\phi)\sin\theta \\
(R + r\cos\phi)\cos\theta \\
0
\end{pmatrix}, \quad
\frac{\partial X}{\partial \phi} = \begin{pmatrix}
-r\sin\phi\cos\theta \\
-r\sin\phi\sin\theta \\
r\cos\phi
\end{pmatrix}.
$$
If we have a point cloud sampled from the manifold, we can estimate the tangent space using Principal Component Analysis (PCA) on local neighborhoods, which is data-driven and does not need parametrization.

In our torus $\mathbb{T}^2$ example, for each query point $x$:
\begin{enumerate}
  \item Find the $k$ nearest neighbors $\{x_1, \ldots, x_k\}$ in the data cloud $X_{\mathrm{tar}}$.
  
  \item Center the neighbors:
  \[
  x_i^{\text{centered}} = x_i - \operatorname{mean}(\{x_1, \ldots, x_k\}).
  \]
  
  \item Compute the local covariance matrix:
  \[
  C = \frac{1}{k} \sum_{i=1}^k x_i^{\text{centered}} (x_i^{\text{centered}})^{\top}.
  \]
  
  \item Perform eigendecomposition:
  \[
  C = V \Lambda V^{\top}.
  \]
  
  \item Take the top $n$ eigenvectors as a tangent basis:
  \[
  T_x \mathcal{M} \approx \operatorname{span}\{v_1, \ldots, v_n\}.
  \]
\end{enumerate}
Since a smooth manifold is locally well approximated by its tangent space, nearby data points lie approximately in a low-dimensional linear subspace. PCA captures the directions of largest variance, which align with the tangent directions of the manifold, whereas small eigenvalues correspond to directions orthogonal to the manifold or variations induced by curvature. However, the approximation error depends on neighborhood size $k$ and curvature. It is also sensitive to noise, thus particles may drift off manifold after a few numerical iterations.

\subsection{1-Sphere $S^1$}\label{app:s_1}
Now we evaluate the proposed KSWGD framework on another example of 1-dimensional spherical manifold $S^1$, targeting a uniform distribution over the unit circle. The data-driven spectral learning utilized $N=500$ training samples $\{\bm{x}^i\}_{i=1}^N$ drawn from this target where $\bm{x}^i=(x_1^i,x_2^i)$ is the i-th sample. To generate training data of time-evolution snapshot pairs $\{(\bm{x}^i, \hat{\bm{x}}^i_{\Delta t})\}_{i=1}^N$ without knowing the potential $V$, we simulated dynamics over a short interval $\Delta t = 0.05$, approximating the drift term $\nabla \log \pi$ via Kernel Density Estimation (KDE) as shown in Appendix~\ref{app:kde} \citep{parzen1962estimation}. Notice that $\hat{\bm{x}}^i_{\Delta t}$ is not exactly from data, rather an augmented data from simulating the Langevin SDE as in~\eqref{eq:sde}. Based on these pairs, the Koopman operator was approximated via Kernel-EDMD \cite{kevrekidis2016kernel} using both a \textit{Polynomial} kernel of degree 10 (see Figure~\ref{fig:sphere_kswgd_rbf}) and a \textit{Gaussian RBF} kernel (see Figure~\ref{fig:sphere_kswgd_polynomial}), using Tikhonov regularization with $\gamma = 10^{-6}$ for numerical stability.

For the generative process, we initialized $M=700$ particles concentrated at the top of the 1-sphere; specifically, $y > 0.7$ depicted in red dots, and evolved them according to the learned Koopman spectral gradients (i.e., KSWGD) for $T=1000$ iterations with a step size of $h=2$. As shown in Figure~\ref{fig:sphere_kswgd}, KSWGD successfully push the particles out and uniformly cover the manifold, as depicted in the purple circles. In contrast, a comparative analysis using the Diffusion Map Particle System (DMPS) as a baseline \citep{li2025diffusion} under identical settings fails to effectively cover the manifold. The result is presented in Figure~\ref{fig:sphere_dmps}. The Algorithm~\ref{alg:kswgd_static} is given below.
\begin{algorithm}[tb]
  \caption{KSWGD for Static Time Data}
  \label{alg:kswgd_static}
  \begin{algorithmic}
    \STATE {\bfseries Input:} training samples $\{z^{(j)}\}_{j=1}^N \sim \pi$, initial particles $\{x_0^{(i)}\}_{i=1}^M$, step size $h$, time step $\Delta t$, KDE bandwidth $\sigma$, truncation rank $r$, dictionary size $n$, max iterations $T$.
    \STATE Estimate score function $\nabla \log \widehat{\pi}(x)$ using KDE.
    \FOR{$j = 1$ {\bfseries to} $N$}
        \STATE Sample $\xi^{(j)} \sim \mathcal{N}(0, I_d)$.
        \STATE $\hat{z}_{\Delta t}^{(j)} \gets z^{(j)} - \Delta t \cdot \nabla \log \widehat{\pi}(z^{(j)}) + \sqrt{2\Delta t}\,\xi^{(j)}$.
    \ENDFOR
    \STATE Construct dictionary $\{\psi_k\}_{k=1}^n$ and estimate the Koopman operator using pairs $\{(z^{(j)}, \hat{z}_{\Delta t}^{(j)})\}_{j=1}^N$.
    \STATE Compute leading $r$ eigenpairs $\{(\widehat{\lambda}_k, \widehat{\phi}_k)\}_{k=1}^r$ of $-\widehat{\mathcal{A}}_N$.
    \FOR{$t = 0$ {\bfseries to} $T-1$}
        \FOR{$i = 1$ {\bfseries to} $M$}
            \STATE Initialize $\mathbf{v}_i \gets \mathbf{0} \in \mathbb{R}^d$.
            \FOR{$j = 1$ {\bfseries to} $M$}
                \STATE $\mathbf{v}_i \gets \mathbf{v}_i + \sum_{k=1}^r 
                \frac{\nabla \widehat{\phi}_k(x_t^{(i)}) \cdot \widehat{\phi}_k(x_t^{(j)})}
                {\widehat{\lambda}_k}$.
            \ENDFOR
            \STATE $x_{t+1}^{(i)} \gets x_t^{(i)} - \frac{h}{M} \mathbf{v}_i$.
        \ENDFOR
    \ENDFOR
    \STATE {\bfseries Output:} generated particles $\{x_T^{(i)}\}_{i=1}^M$.
  \end{algorithmic}
\end{algorithm}

Next we also compare KSWGD to DMPS on $S^1$, as shown in Figure~\ref{fig:sphere_dmps}. Under the same experimental settings, DMPS fails to reach full convergence to the target distribution, even when the number of iterations or the step size is moderately increased. This further highlights the advantage of KSWGD’s spectral preconditioning in sustaining an effective gradient flow throughout the optimization.

%%%%%%%% %%%%%%%%
\subsection{KDE-Based Score Estimation}\label{app:kde}

The score function $\nabla \log \pi(x)$ is estimated using a Gaussian kernel density estimator as following:

First we compute pairwise squared distances:
\begin{equation*}
    D_{ij}^2 = \|x_i - x_j\|^2, \quad i,j = 1,\ldots,n.
\end{equation*}
Then we select bandwidth  by median heuristic: $h = \sqrt{\text{median}(D^2)}$. Next, we compute the Gaussian kernel weights:
\begin{equation*}
    W_{ij} = \exp\left(-\frac{\|x_i - x_j\|^2}{2h^2}\right),
\end{equation*}
and normalize weights $w_{ij} = \frac{W_{ij}}{\sum_{k=1}^{n} W_{ik}}$. Then we compute weighted mean $\bar{x}_i = \sum_{j=1}^{n} w_{ij} x_j$. The score function estimation is give as follows:
\begin{equation*}
    \nabla \log \hat{\pi}(x_i) = \frac{\bar{x}_i - x_i}{h^2} = \frac{1}{h^2}\left(\frac{\sum_{j=1}^{n} W_{ij} x_j}{\sum_{j=1}^{n} W_{ij}} - x_i\right).
\end{equation*}

This estimator approximates the gradient of the log-density by computing the direction from each point $x_i$ toward its local kernel-weighted centroid, scaled by the squared bandwidth $h^2$.

%%%%%%%%%%%%%%%%%%%%%%%%%%%
\subsection{Koopman Operator Prediction in Latent Space}\label{app:latent_prediction}
To predict the time evolution of latent dynamics in the Stochastic Allen-Cahn equation, we employ an Extended Dynamic Mode Decomposition (EDMD) approach with polynomial dictionary functions. Let $z \in \mathbb{R}^d$ denote the latent vector. The central idea is to lift the latent representation into a higher-dimensional feature space where the dynamics become approximately linear.

We define a polynomial feature map $\Phi: \mathbb{R}^d \to \mathbb{R}^{N_K}$ that includes all monomials up to degree $p$:
$$
\Phi(z) \coloneqq [\underbrace{1}_{\text{constant}}, \underbrace{z_1, \ldots, z_d}_{\text{coordinates}}, \underbrace{z_1^2, z_1 z_2, \ldots, z_1^p, \ldots, z_d^p}_{\text{higher-order terms up to degree $N_K$}}]
$$
Crucially, the original latent coordinates $[z_1, \ldots, z_d]$ appear explicitly as second through $(d+1)$-th entries in this dictionary $\Phi(z)$. Given paired trajectory data $\{(z_i^{(0)}, z_i^{(\Delta t)})\}_{i=1}^N$ sampled at times $t=0$ and $t=\Delta t$, we estimate the finite-dimensional Koopman operator $K \in \mathbb{R}^{N_K \times N_K}$ by solving the least-squares problem:
$$
\Phi(z^{(\Delta t)}) \approx \Phi(z^{(0)}) \, K.
$$

For multi-step prediction, we first compute its lifted representation $\Phi(z_t)$ starting from a latent state $z_t$, then propagate forward in the feature space via $\Phi_{t+\Delta t} = \Phi(z_t) K$. To recover the predicted latent state, we simply extract the $[z_1, \ldots, z_d]$ part from $\Phi_{t+\Delta t}$, which gives the updated coordinates directly. This operation is repeated to generate trajectories over arbitrary time steps, which enables the Koopman operator to capture the nonlinear latent dynamics through the polynomial dictionary.

%%%%%%%%%%%%%%%%%%%%%%%%%%%
\subsection{Feynman-Kac Interpretation of KSWGD}\label{app:feynman-kac} 

The Koopman operator employed in KSWGD admits a natural interpretation within the Feynman-Kac framework \citep{bertini1995stochastic,karatzas2014brownian,oksendal2003stochastic}. Recall that the general Feynman-Kac formula provides the solution to the following equation:
$$\partial_t v = \mathcal{A} v - U(x) v, \quad v(x,0) = f(x),$$
as the path integral:
$$v(x,t) = \mathbb{E}\left[f(X_t) \exp\left(-\int_0^t U(X_s) ds\right)\mid X_0=x\right],$$
where $U(x)$ acts as a \textit{killing} (or potential) term and $\mathcal{A}$ is the Kolmogorov backward operator (i.e., the infinitesimal generator of the Koopman semigroup associated with the underlying stochastic dynamics as discussed in Section~\ref{sec:koopman}). 

\paragraph{Probabilistic interpretation of the killing rate.}
The terminology ``killing rate'' arises from a probabilistic interpretation; more specifically, considering a particle following the stochastic trajectory $X_s$, at each position $x$, the particle is ``killed'' (removed from the system) at rate $U(x)$. The exponential term $\exp\left(-\int_0^t U(X_s) \, ds\right)$ then represents the survival probability of the particle along its path up to time $t$. When $U(x) > 0$, paths passing through regions of high $U$ are down-weighted; when $U \equiv 0$, all paths contribute equally.

\paragraph{What does $U = 0$ mean in KSWGD?}
The Koopman semigroup $\{\mathcal{T}^t\}_{t\geq 0}$ used in KSWGD corresponds to the zero-potential case $U \equiv 0$:
$$v(x,t) = (\mathcal{T}^t f)(x) = \mathbb{E}[f(X_t)\mid X_0=x].$$
This choice is not an approximation but rather the mathematically appropriate choice for our task. Importantly, the potential $U$ in the Feynman-Kac formula is distinct from the drift potential $V$ that defines the target distribution $\pi \propto e^{-V}$; the drift potential $V$ is already encoded in the Langevin generator $\mathcal{L} = -\mathcal{A} = \nabla V \cdot \nabla - \Delta$. Setting $U = 0$ means that we compute unconditional expectations over paths, which is precisely what is needed for sampling the marginal distribution at $t=0$. Specifically, KSWGD targets the distribution of initial conditions $\pi_0$ without conditioning on future behavior, and the pushforward $(\mathcal{T}^t)_\# \pi_0$ gives the marginal distribution at time $t$ without path conditioning.

\paragraph{Extension to conditional sampling.}
In contrast, $U \neq 0$ would be appropriate for conditional sampling problems, e.g., sampling initial conditions that lead to a specific terminal state such as rare event sampling \citep{weinan2021applied} or importance-weighted path integrals for computing conditional expectations. If one seeks initial conditions that evolve toward a specific target set $B\subset\mathbb{R}^d$, the framework naturally extends to $U \neq 0$ via importance sampling:
$$\tilde{\pi}_0(x) \propto \pi_0(x) \cdot \mathbb{E}\left[\mathbf{1}_B(X_T) \exp\left(-\int_0^T U(X_s) ds\right)\mid X_0=x\right].$$
However, detailed analysis and development of the framework for $U \neq 0$ are beyond the scope of the current study and represent a promising direction for future research.

\begin{remark}
The linear Feynman-Kac formula may appear to conflict with nonlinear partial differential equations such as the Allen-Cahn equation in Section~\ref{sec:allen_cahn}. However, these two equations describe different objects. The Allen-Cahn equation governs the evolution of the state $u(x,t)$ and is nonlinear in $u$. In contrast, the Feynman-Kac PDE governs the evolution of expected observables $v(x,t) = (\mathcal{T}^t f)(x)$, which is linear in $f$. This is precisely the essence of Koopman operator theory; more specifically, a nonlinear dynamical system acting on states induces a linear operator acting on observables. See \citep[Section 8.2]{oksendal2003stochastic} for more details.
\end{remark}

%%%%%%%%%%%%%%%%%%%%%%%%%%%
\subsubsection{Stochastic Allen-Cahn Equation}\label{app:stochastic_ac_eqn}

We simulate the stochastic Allen-Cahn equation of Eq.~\eqref{eq:ac} on a $128 \times 128 = 16,384$ grid with time step $\Delta t = 10^{-5}$, diffusion coefficient $D=0.001$, interface width $\epsilon=0.01$, and noise intensity $\sigma=\sqrt{2}$. The small $\epsilon$ drives rapid reaction dynamics, necessitating fine temporal resolution. We generate 300 independent realizations with snapshots recorded at $t \in \{0, 0.0001, 0.0002, 0.0005\}$. Each snapshot is compressed to a 16-dimensional latent space via a fully-connected autoencoder with architecture $16384 \to 512 \to 128 \to 16$ (encoder) and symmetric decoder.

\paragraph{KSWGD parameters.} The Koopman matrix $K$ is constructed using EDMD with 2nd-order polynomial dictionary functions from paired data snapshot at $t_0=0$ and $t_1=0.0001$, with regularization parameter $10^{-4}$. During sampling, 150 particles are initialized in the latent space from a Gaussian distribution rescaled to match the statistics of the latent encodings at $t_0$, and evolved over 800 iterations with step size $h=0.03$.

\paragraph{Baseline methods.} To ensure fair comparison as in Figure~\ref{fig:ac_comparison1}, all generative baselines operate in the same shared 8-dimensional latent space. Each baseline is trained for 150 epochs using the Adam optimizer and generates 100 samples per time point. Unlike KSWGD, which learns dynamics from $t_0 \to t_1$ pairs and predicts via powers of $K$, each baseline model is trained separately on ground-truth samples at each time point.
\begin{itemize}
    \item \textbf{Diffusion Model (DDPM):} MLP-based noise predictor with hidden dimension 256 and sinusoidal time embedding of dimension 64. Linear beta schedule from $10^{-4}$ to $0.02$ over 200 diffusion steps.
    \item \textbf{VAE:} Second-level encoder-decoder mapping 8-dimensional latent codes to a 4-dimensional VAE latent space. Two hidden layers with dimension 128 and KL weight $\beta = 0.1$.
    \item \textbf{Normalizing Flows (RealNVP):} 4 affine coupling layers with alternating binary masks, each parameterized by MLPs with hidden dimension 128.
    \item \textbf{GAN (WGAN-GP):} Generator maps 16-dimensional noise to 8-dimensional latent codes. 5 critic updates per generator update with gradient penalty $\lambda = 10.0$.
\end{itemize}

\paragraph{Additional experiments.} We also present results with interface width parameter $\epsilon=0.02$ in Figure~\ref{fig:ac_21} and Figure~\ref{fig:ac_comparison1}. Smaller $\epsilon$ creates sharper phase interfaces approaching the sharp-interface limit, requiring finer spatial resolution and smaller timesteps for numerical stability.

\begin{figure}[!htbp]
    \centering

    \begin{subfigure}{0.95\linewidth}
        \centering
        \includegraphics[width=\linewidth]{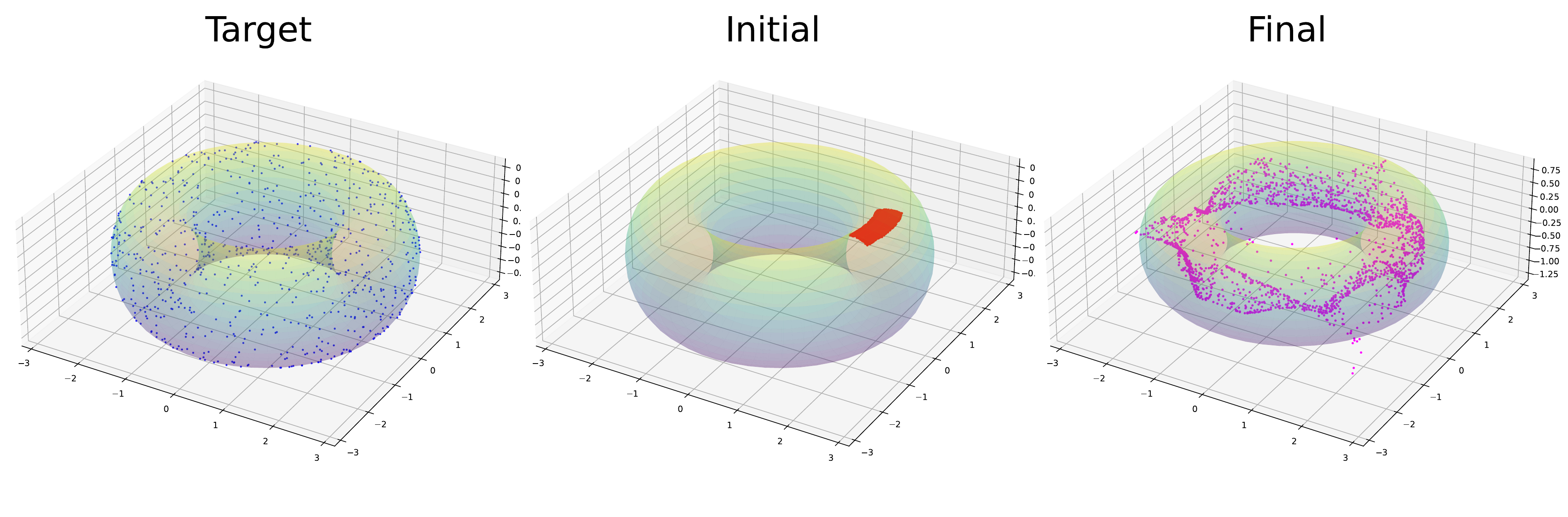}
        \caption{Three-dimensional visualization of particle evolution on the torus $\mathbb{T}^2$ with major radius $R=2.0$ and minor radius $r=0.8$. \textbf{(Left)} Target distribution. \textbf{(Center)} Initial configuration. \textbf{(Right)} Final configuration after DMPS iterations.}
    \end{subfigure}

    \vspace{0.3cm}

    \begin{subfigure}{0.75\linewidth}
        \centering
        \includegraphics[width=\linewidth]{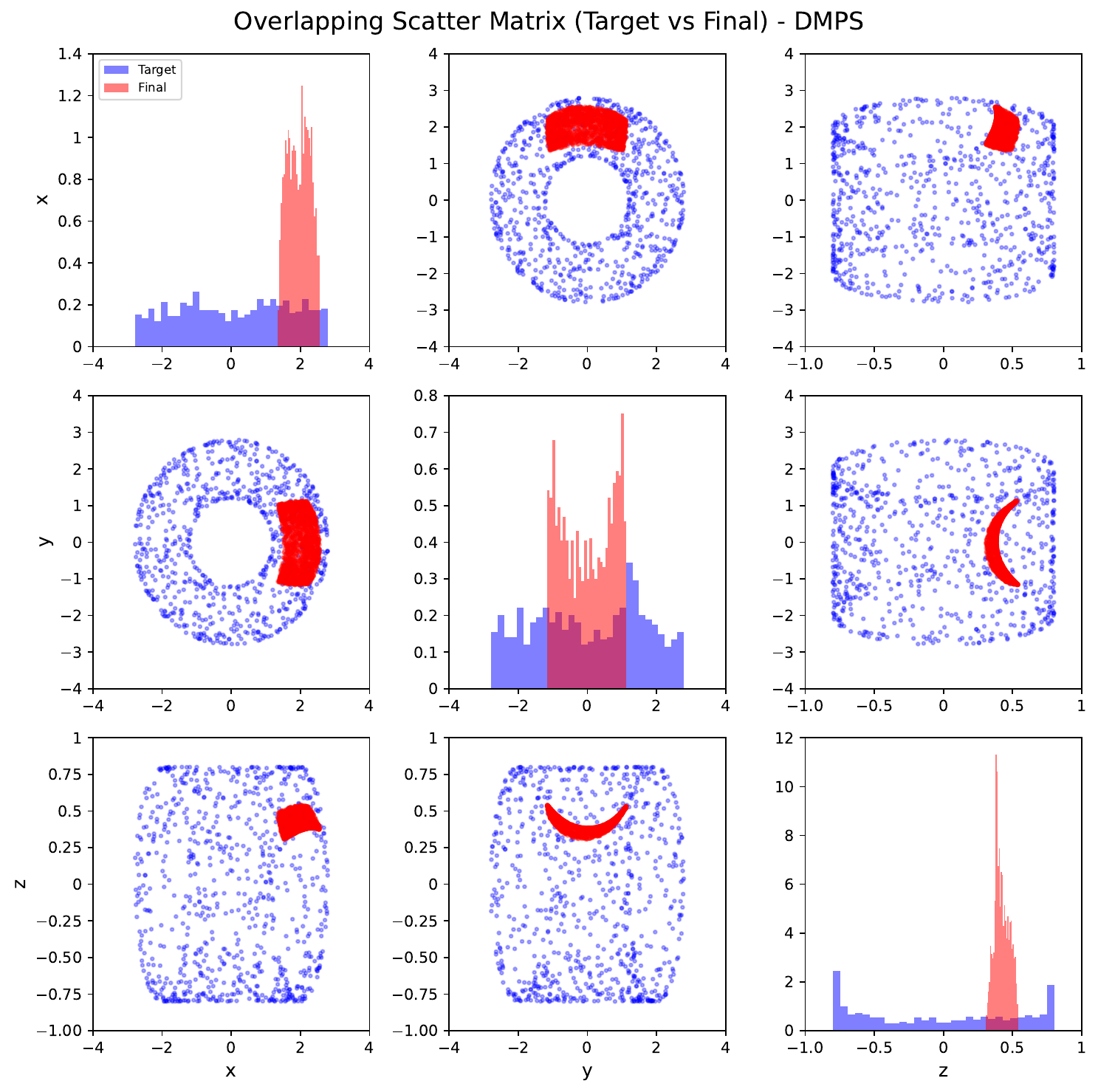}
        \caption{Scatter matrix comparing target (blue) and final particle distribution (red).}
    \end{subfigure}

    \caption{DMPS: Particle distribution convergence on torus $\mathbb{T}^2$ shown through 3D evolution and pairwise statistical comparison.}
    \label{fig:torus_dmps}
\end{figure}

\begin{figure}[!htbp]
    \centering
    \resizebox{0.75\linewidth}{!}{%
    \begin{minipage}{\linewidth}
        \centering

        % --- Subfigure (a) ---
        \begin{subfigure}{0.48\linewidth}
            \centering
            \includegraphics[width=\linewidth]{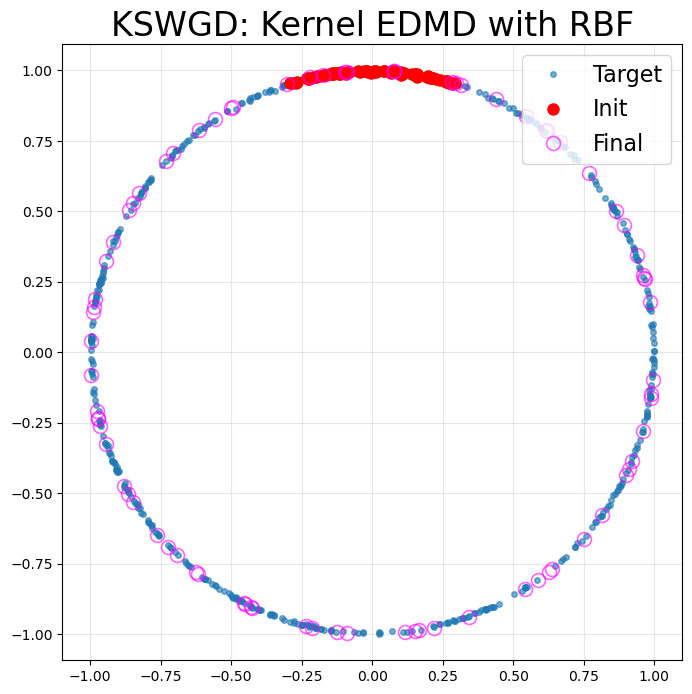}
            \caption{Kernel: RBF}
            \label{fig:sphere_kswgd_rbf}
        \end{subfigure}
        %
        % --- Subfigure (b) ---
        \begin{subfigure}{0.48\linewidth}
            \centering
            \includegraphics[width=\linewidth]{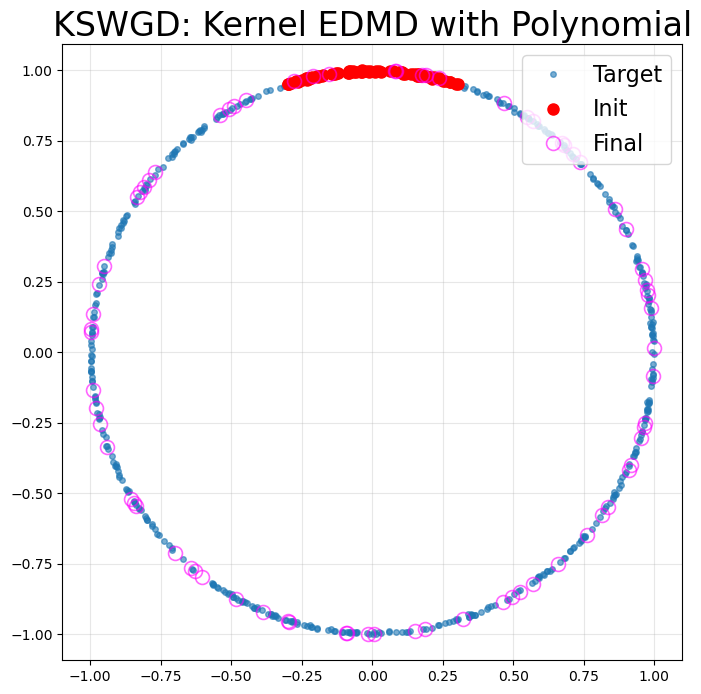}
            \caption{Kernel: Polynomial}
            \label{fig:sphere_kswgd_polynomial}
        \end{subfigure}
    \end{minipage}
    }% end resizebox
    \caption{1-Sphere $S^1$ example using KSWGD with Kernel-EDMD.}
    \label{fig:sphere_kswgd}
\end{figure}

\begin{figure}[!htbp]
    \centering
    \resizebox{0.75\linewidth}{!}{%
    \begin{minipage}{\linewidth}
        \centering

        % --- Subfigure (a) ---
        \begin{subfigure}{0.48\linewidth}
            \centering
            \includegraphics[width=\linewidth]{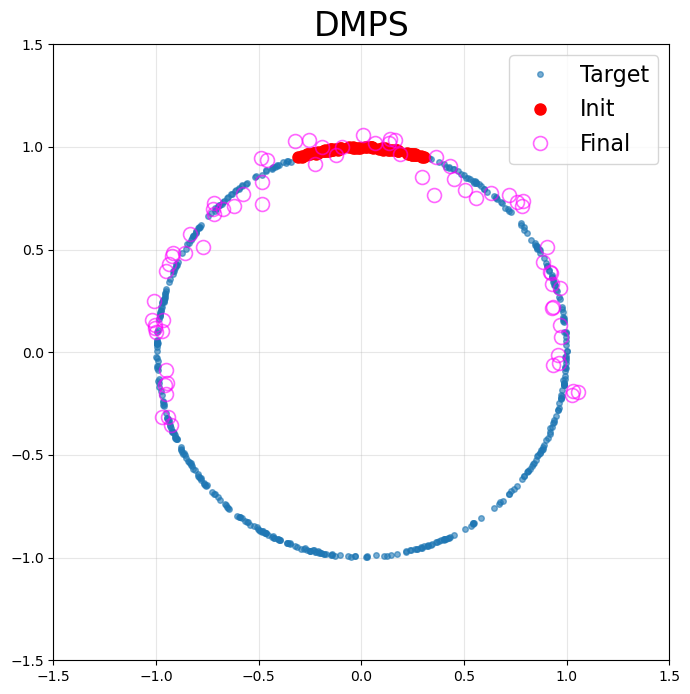}
        \end{subfigure}
        %
        % --- Subfigure (b) ---
        \begin{subfigure}{0.48\linewidth}
            \centering
            \includegraphics[width=\linewidth]{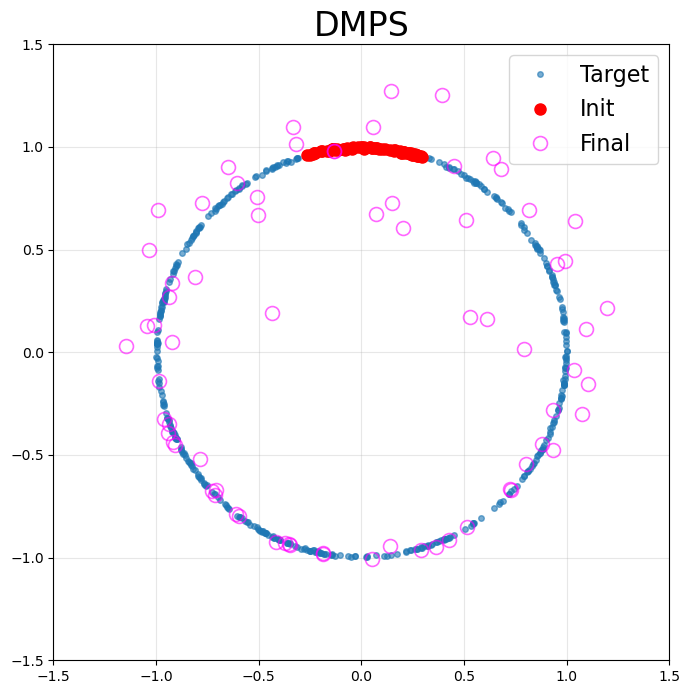}
        \end{subfigure}

    \end{minipage}
    }% end resizebox
    \caption{1-Sphere $S^1$ example using DMPS with different step size $h$.}
    \label{fig:sphere_dmps}
\end{figure}

\begin{figure}[!htbp]
    \centering
    \resizebox{0.85\linewidth}{!}{%
    \begin{minipage}{\linewidth}
        \centering

        % --- Subfigure (a) ---
        \begin{subfigure}{0.95\linewidth}
            \centering
            \includegraphics[width=\linewidth]{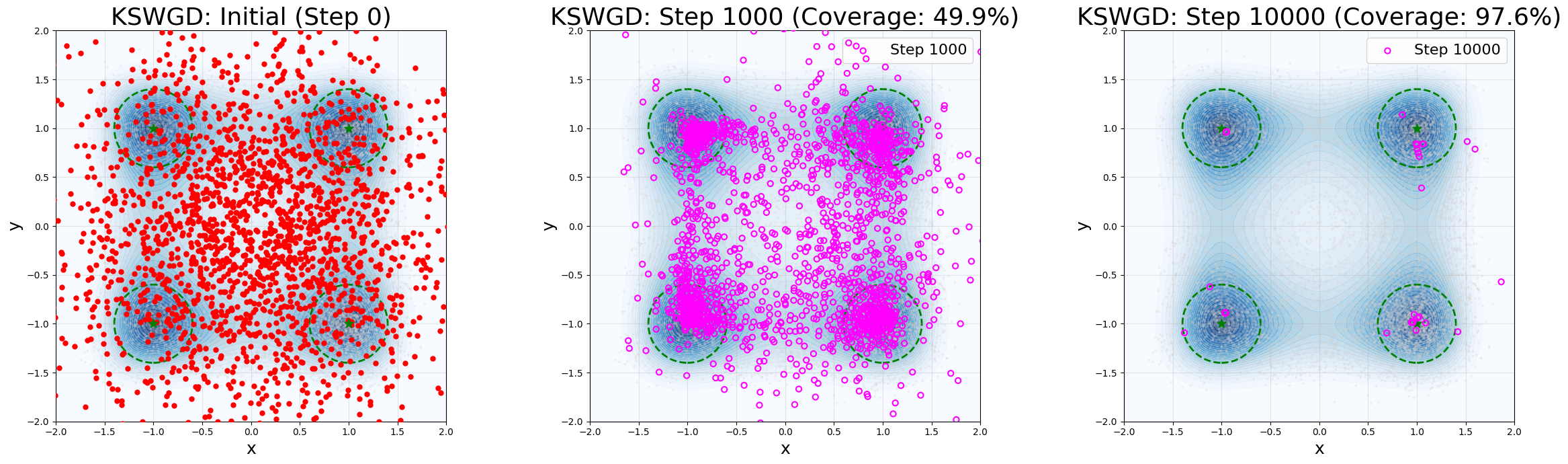}
            % \caption{}
        \end{subfigure}
        %
        % --- Subfigure (b) ---
        \begin{subfigure}{0.95\linewidth}
            \centering
            \includegraphics[width=\linewidth]{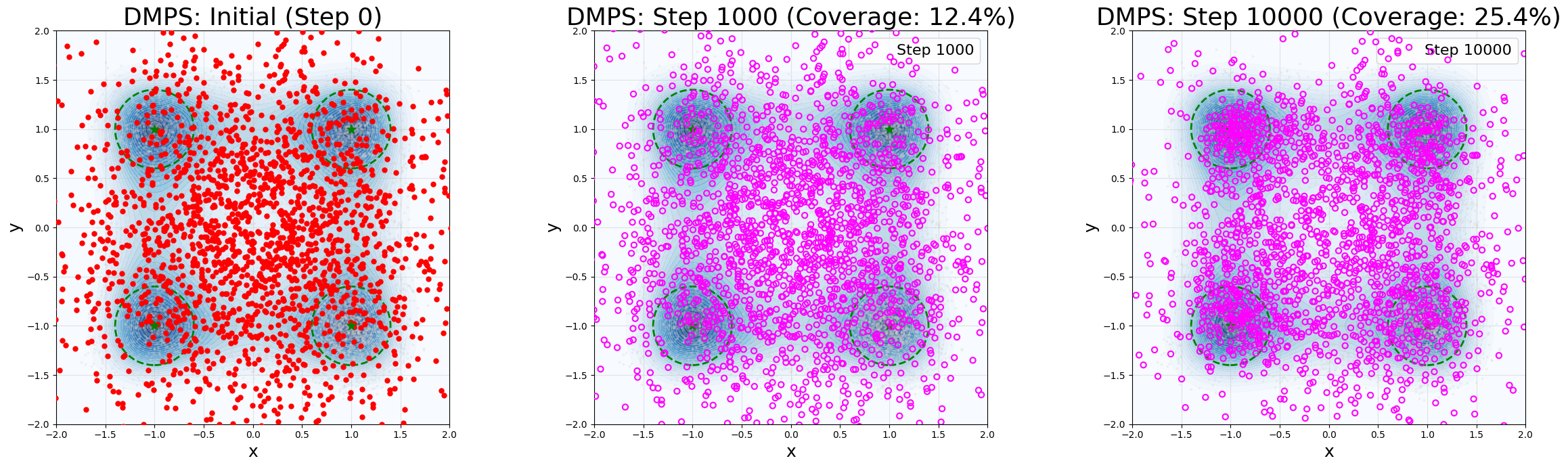}
            % \caption{}
        \end{subfigure}
    \end{minipage}
    }% end resizebox
    \caption{Quadruple potential well example using KSWGD and DMPS.}
    \label{fig:quadruple_kswgd_dmps}
\end{figure}
\begin{figure}[!htbp]
    \centering
    \resizebox{0.85\linewidth}{!}{%
    \begin{minipage}{\linewidth}
        \centering

        % --- Subfigure (a) ---
        \begin{subfigure}{0.49\linewidth}
            \centering
            \includegraphics[width=\linewidth]{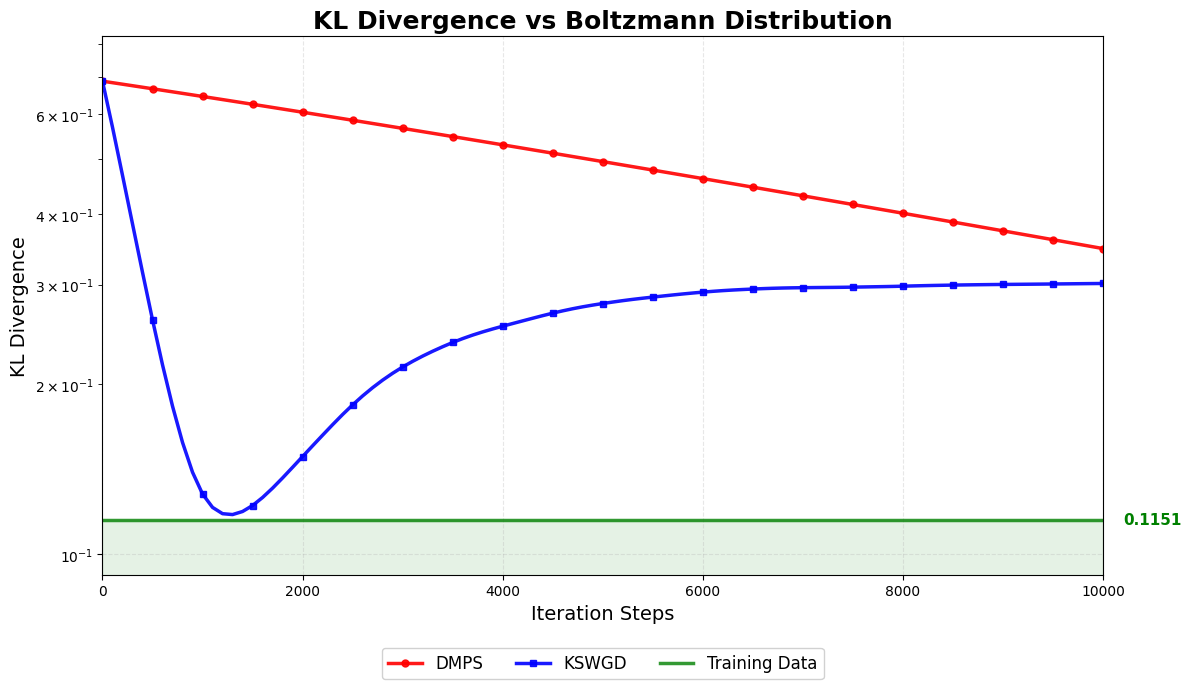}
            \caption{}
            \label{fig:kl2}
        \end{subfigure}
        %
        % --- Subfigure (b) ---
        \begin{subfigure}{0.49\linewidth}
            \centering
            \includegraphics[width=\linewidth]{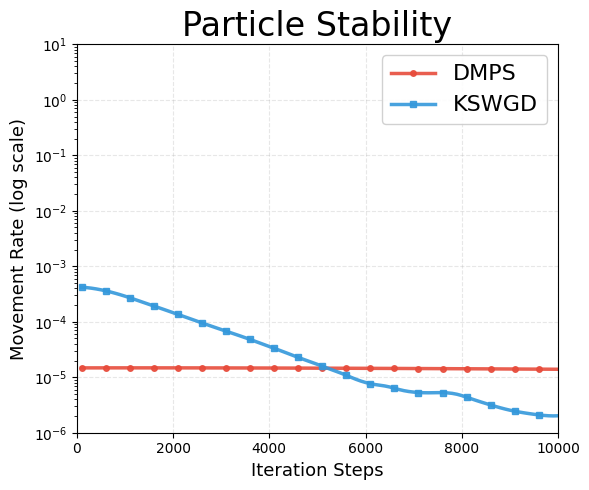}
            \caption{}
        \end{subfigure}
    \end{minipage}
    }% end resizebox
    \caption{KL divergence of DMPS, KSWGD and KDE of training data.}
    \label{fig:kl}
\end{figure}
\begin{figure}[!htbp]
    \centering
    \resizebox{0.85\linewidth}{!}{%
    \begin{minipage}{\linewidth}
        \centering

        % --- Subfigure (a) ---
        \begin{subfigure}{0.95\linewidth}
            \centering
            \includegraphics[width=\linewidth]{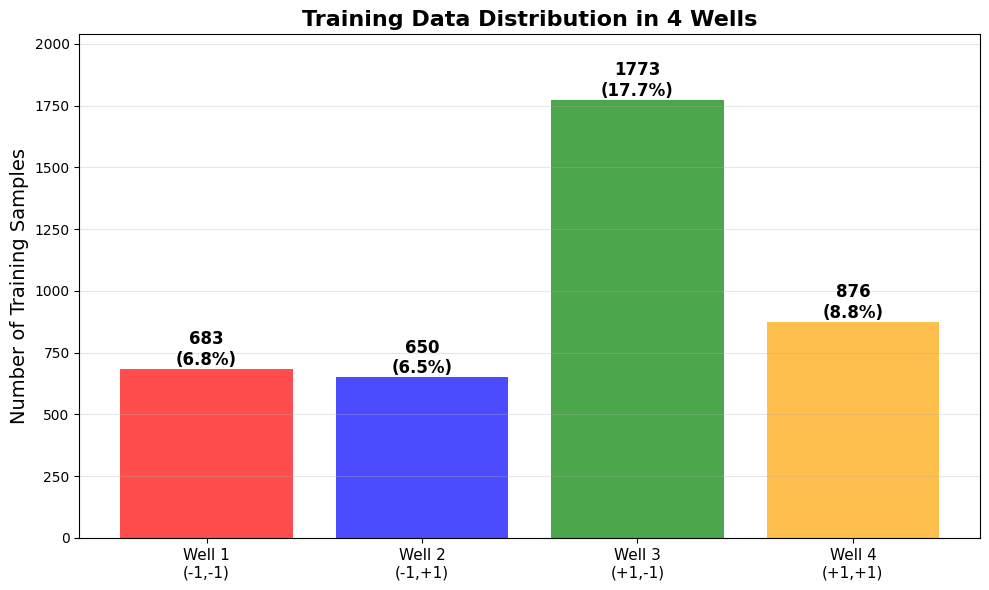}
            % \caption{}
        \end{subfigure}
        %
        % --- Subfigure (b) ---
        \begin{subfigure}{0.95\linewidth}
            \centering
            \includegraphics[width=\linewidth]{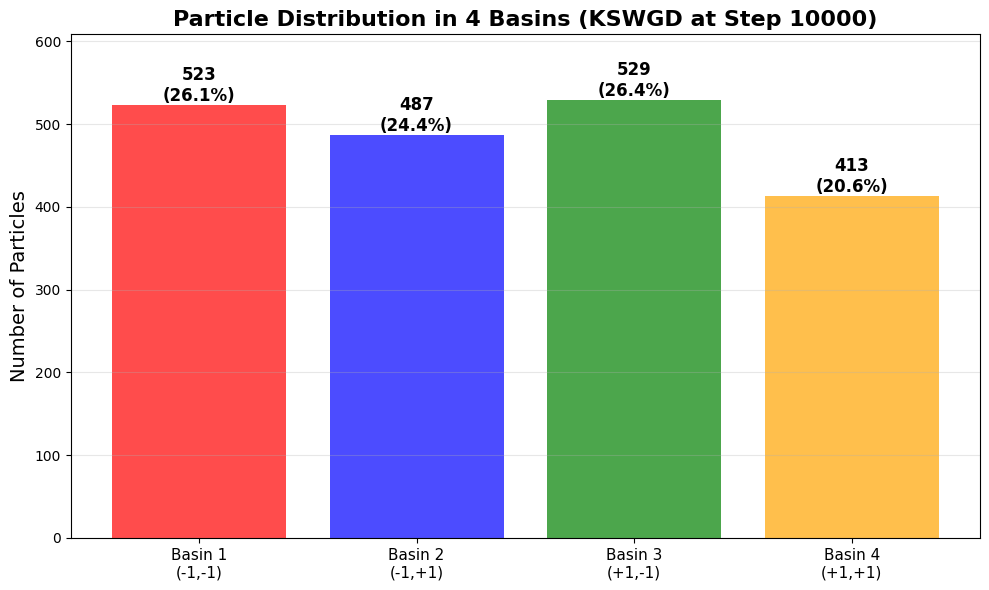}
            % \caption{}
        \end{subfigure}
    \end{minipage}
    }% end resizebox
    \caption{The distribution of the training data and after 10000 steps by KSWGD at each well.}
    \label{fig:distribution_4_wells}
\end{figure}

\begin{figure}[!htbp]
    \centering
    \resizebox{0.85\linewidth}{!}{%
    \begin{minipage}{\linewidth}
        \centering

        % --- Subfigure (a) ---
        \begin{subfigure}{0.48\linewidth}
            \centering
            \includegraphics[width=\linewidth]{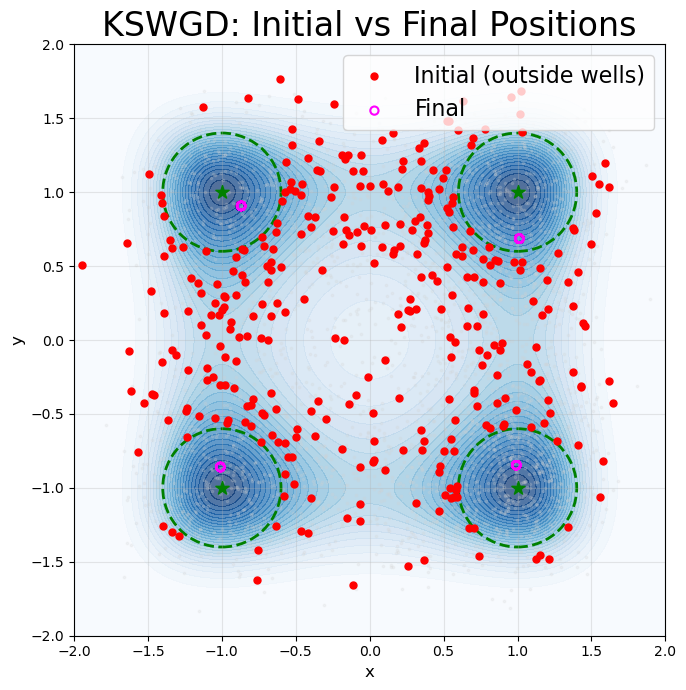}
            % \caption{}
        \end{subfigure}
        %
        % --- Subfigure (b) ---
        \begin{subfigure}{0.48\linewidth}
            \centering
            \includegraphics[width=\linewidth]{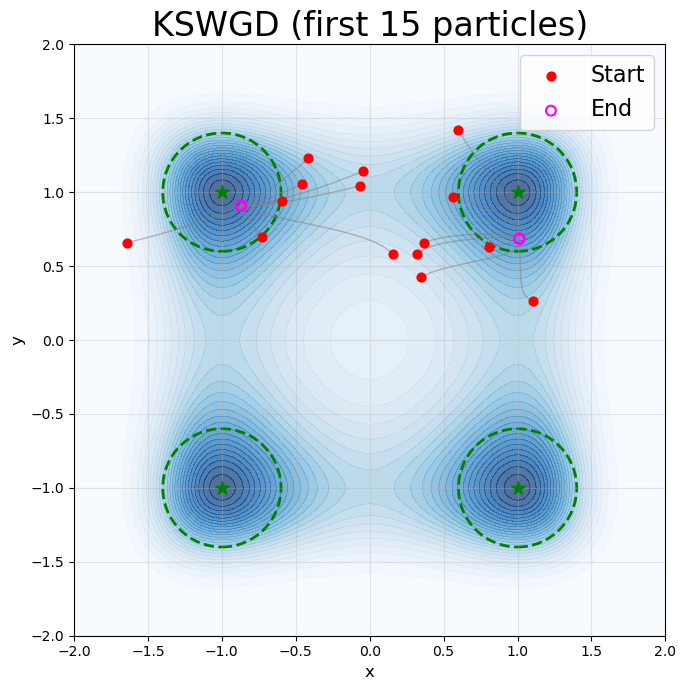}
            % \caption{}
        \end{subfigure}
        %
        % --- Subfigure (c) ---
        \begin{subfigure}{0.48\linewidth}
            \centering
            \includegraphics[width=\linewidth]{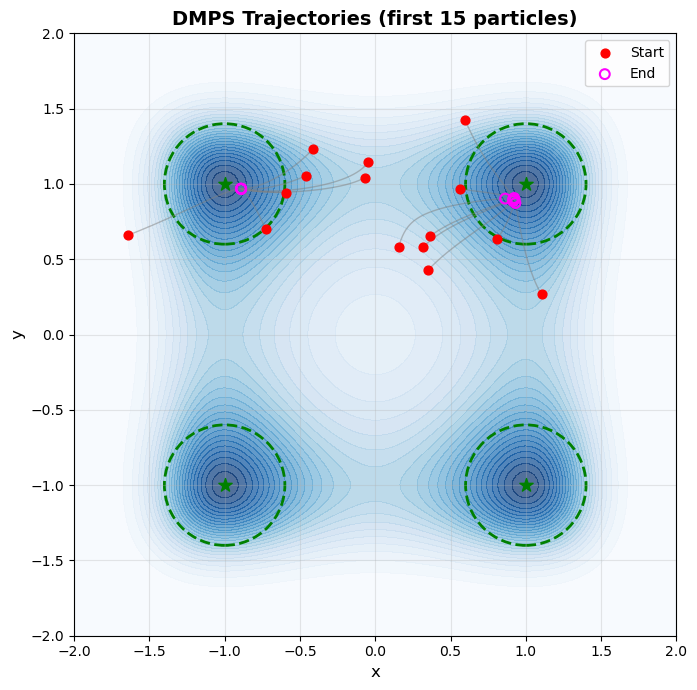}
            % \caption{}
        \end{subfigure}
        %
        % --- Subfigure (d) ---
        \begin{subfigure}{0.48\linewidth}
            \centering
            \includegraphics[width=\linewidth]{images/quadruple_well/dmps_1000_time_series2.png}
            % \caption{}
        \end{subfigure}
        %
        % --- Subfigure (e) ---
        \begin{subfigure}{0.48\linewidth}
            \centering
            \includegraphics[width=\linewidth]{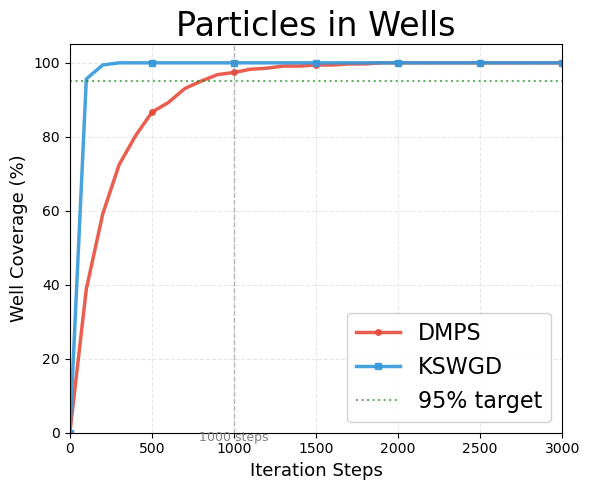}
        \end{subfigure}
        %
        % --- Subfigure (f) ---
        \begin{subfigure}{0.48\linewidth}
            \centering
            \includegraphics[width=\linewidth]{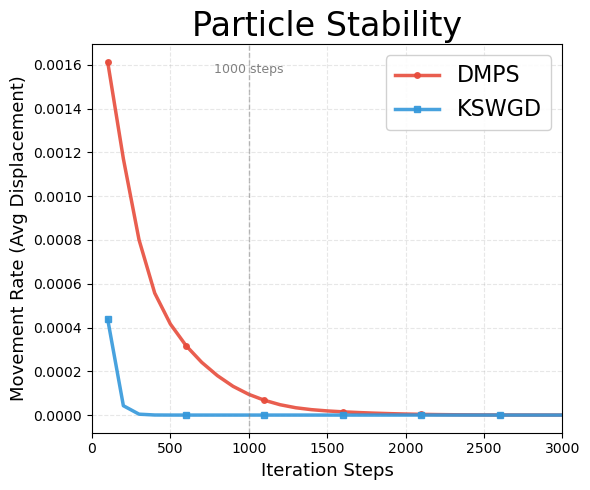}
        \end{subfigure}
    \end{minipage}
    }% end resizebox
    \caption{Quadruple potential well example using KSWGD and DMPS with smaller kernel bandwidth value.}
    \label{fig:quadruple_kswgd2}
\end{figure}

\begin{figure}[!htbp]
    \centering
    \includegraphics[width=0.75\linewidth]{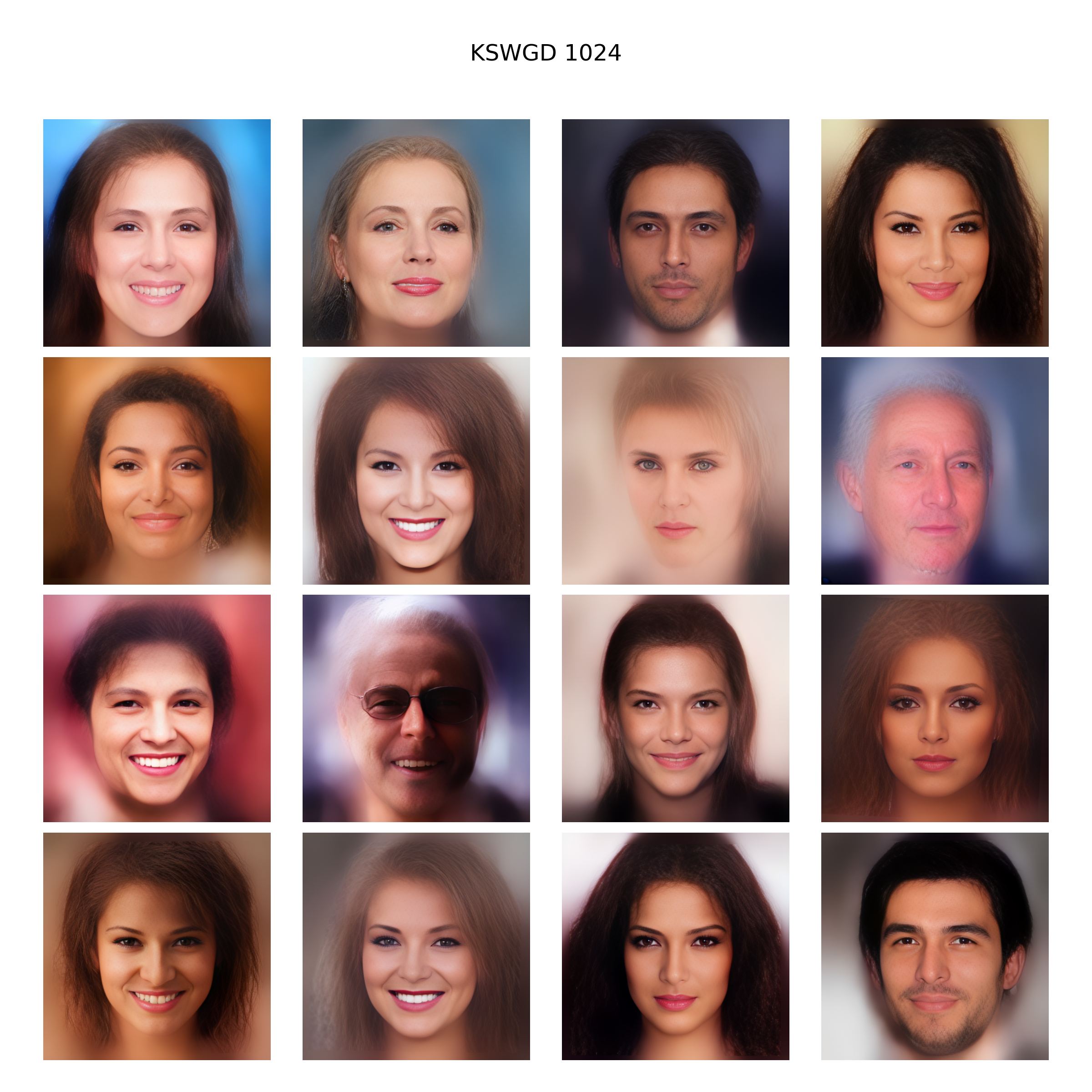}

    \vspace{0.3cm}

    \begin{center}
    \begin{small}
    \begin{tabular}{lcc}
      \toprule
      Data set & LDM-200 & KSWGD \\
      \midrule
      CelebA-256  & 108 & 161 \\
      CelebA-1024 & 117 & 137 \\
      \bottomrule
    \end{tabular}
    \end{small}
    \end{center} 
    \caption{CelebA-HQ and corresponding FID scores.}
    \label{fig:celeba_with_table}
\end{figure}

\begin{figure}[!htbp]
    \centering    \includegraphics[width=0.99\linewidth]{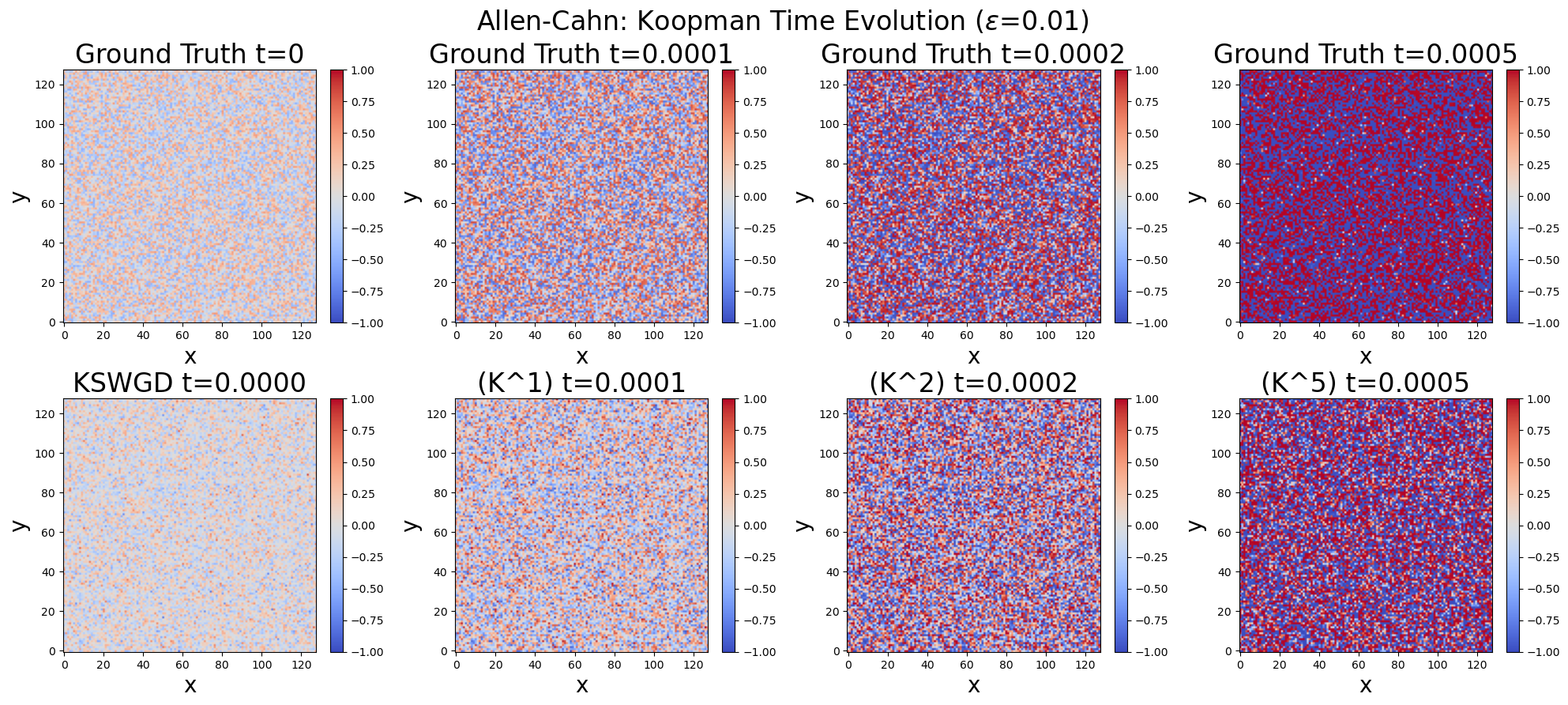}    
    \caption{Solution of Allen-Cahn equation example using KSWGD with EDMD (polynomial dictionary) and $\epsilon = 0.01$.}
    \label{fig:ac_14}
\end{figure}

\begin{figure*}[!htbp]
    \centering
    \includegraphics[width=0.99\linewidth]{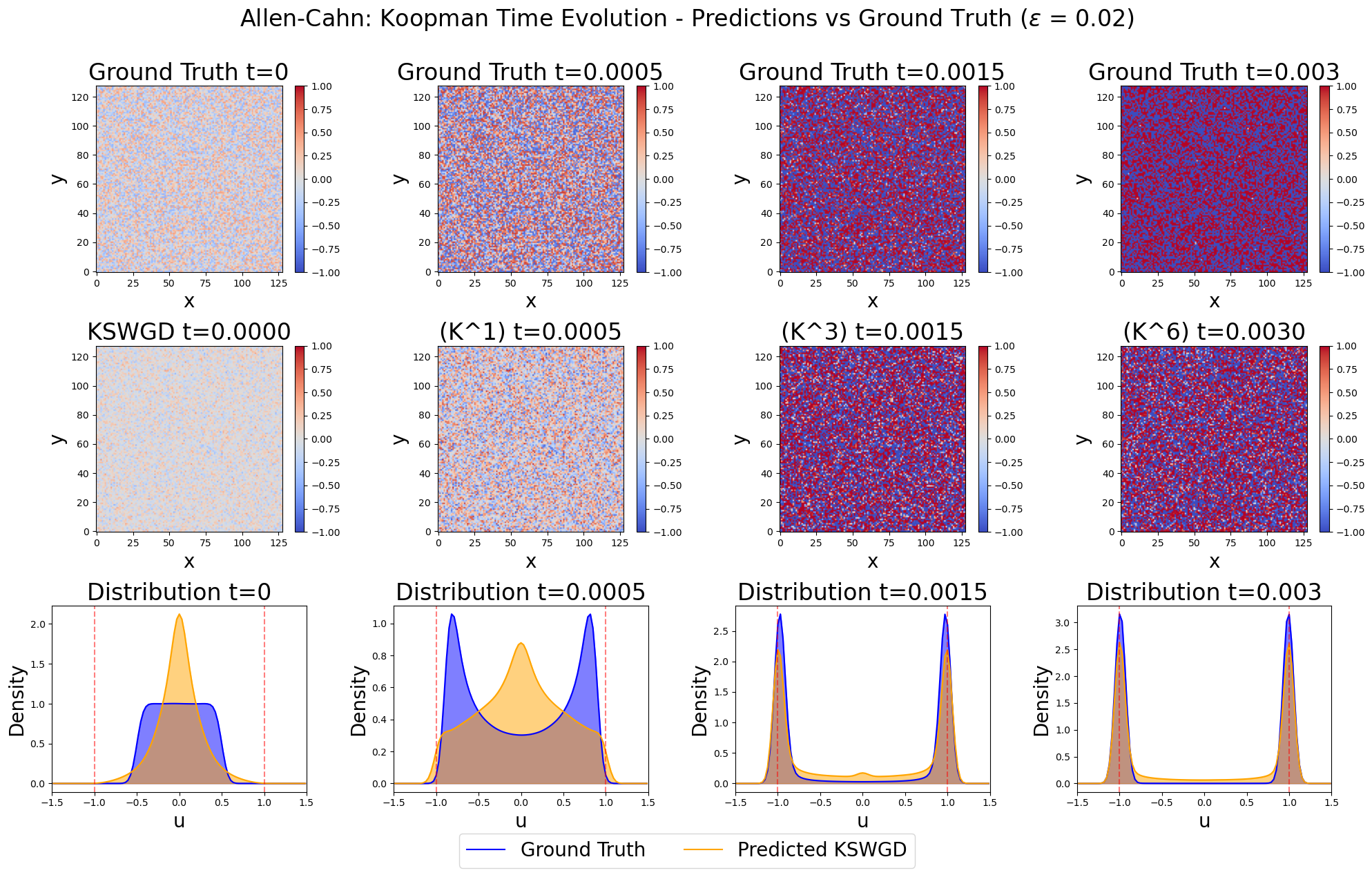}
    \caption{Solution of Allen-Cahn equation example using KSWGD with EDMD (polynomial dictionary) and $\epsilon = 0.02$.}
    \label{fig:ac_21}
\end{figure*}

\begin{figure*}[!htbp]
    \centering
    \resizebox{0.99\linewidth}{!}{%
    \begin{minipage}{\linewidth}
        \centering

        % --- Subfigure (a) ---
        \begin{subfigure}{0.99\linewidth}
            \centering
            \includegraphics[width=0.99\linewidth]{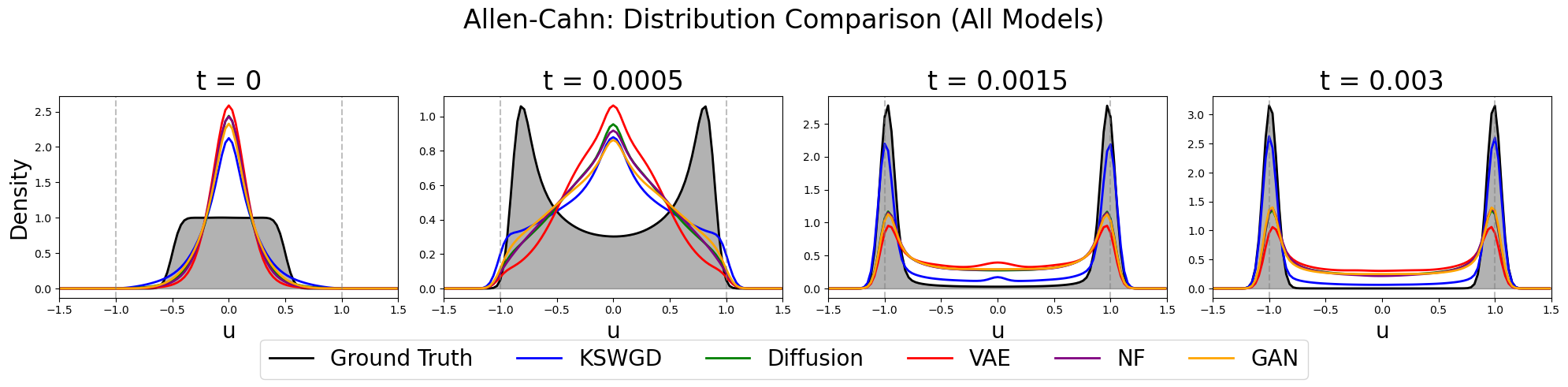}
        \end{subfigure}

    \end{minipage}
    }% end resizebox
    \caption{Comparison: KSWGD vs DM, VAE, Normalizing Flows and GAN, $\epsilon = 0.02$.}
    \label{fig:ac_comparison2}
\end{figure*}

\newpage
%%%%%%%%%%%%%%%%%%%%%%%%%%%%%
\section{Theory}\label{sec:appendixB}
%%%%%%%%%%%%%%%%%%%%%%%%%%%%%
\subsection{Proof of Proposition~\ref{prop:ideal}}\label{app:proof_prop}
\begin{proof}
Let $\rho_t = \frac{\mathrm{d}\mu_t}{\mathrm{d}\pi}$ and $f_t = \rho_t - 1$. 
% Then, at the density level with respect to $\pi$, \eqref{eq:truncated_kpm_LAWGD} implies
% \begin{equation*}
% \partial_t \rho_t = \operatorname{div}(\rho_t \nabla \mathcal{K}_r f_t).
% \end{equation*}
Recall the KL divergence $\mathrm{KL}(\mu_t\|\pi)=\int \rho_t \log\rho_t\,\mathrm{d}\pi$. By differentiation under the invariant measure $\pi$, see \citep{chewi2023log,chewi2020svgd} for more details), we obtain
\begin{align*}
\frac{\mathrm{d}}{\mathrm{d}t} \mathrm{KL}(\mu_t\|\pi)
&= \int \partial_t (\rho_t \log\rho_t) \, \mathrm{d}\pi  \\
&= \int \left[(\partial_t\rho_t)\log\rho_t + \rho_t\partial_t(\log\rho_t)\right] \, \mathrm{d}\pi \\
&= \int \mathrm{div}\!\bigl(\rho_t\,\nabla \mathcal K_r f_t\bigr)\,\log\rho_t \, \mathrm{d}\pi + \int \partial_t \rho_t \, \mathrm{d}\pi \\
&= - \int \big\langle \nabla \log\rho_t,\, \nabla \mathcal K_r f_t \big\rangle \,\rho_t \, \mathrm{d}\pi \\
&= - \big\langle \nabla \log\rho_t,\, \nabla \mathcal K_r f_t \big\rangle_{\mu_t},
\end{align*}
where $\int \partial_t \rho_t \, \mathrm{d}\pi = \frac{\mathrm{d}}{\mathrm{d}t}\int \rho_t \, \mathrm{d}\pi = 0$. Then, using $\rho_t=1+f_t$ we have the identity
\[
\big\langle \nabla \log\rho_t,\, \nabla g\big\rangle_{\mu_t}
= \int \Big\langle \frac{\nabla \rho_t}{\rho_t},\, \nabla g \Big\rangle \rho_t \, \mathrm{d}\pi
= \int \langle \nabla \rho_t,\, \nabla g \rangle \, \mathrm{d}\pi
= \int \langle \nabla f_t,\, \nabla g \rangle \, \mathrm{d}\pi
= \big\langle \nabla f_t,\, \nabla g \big\rangle_\pi .
\]

Applying this with $g=\mathcal K_r f_t$ gives
\begin{equation}\label{eq:DKL_main}
\frac{\mathrm{d}}{\mathrm{d}t} \mathrm{KL}(\mu_t\|\pi)
= - \big\langle \nabla f_t,\, \nabla \mathcal K_r f_t \big\rangle_\pi
= -\langle f_t, \mathcal L \mathcal K_r f_t \rangle_\pi
= -\langle f_t, \Pi_r  f_t \rangle_\pi,
\end{equation}
where we used the Dirichlet form identity $\langle \nabla f, \nabla g\rangle_\pi = \langle f, \mathcal L g\rangle_\pi$ and the operator relation~\eqref{eq:main_proj}.

Expanding the right-hand side of~\eqref{eq:DKL_main}, we have
\begin{equation}\label{eq:DKL_bound_step}
\frac{\mathrm{d}}{\mathrm{d}t} \mathrm{KL}(\mu_t\|\pi)
= - \langle f_t, \Pi_r f_t \rangle_\pi = -\|\Pi_r f_t\|_{L^2_\pi}^2,
\end{equation}
since $\Pi_r$ is the orthogonal projection, which is self-adjoint and idempotent and thus implies $\langle f_t, \Pi_r f_t \rangle_\pi = \|\Pi_r f_t\|_{L^2_\pi}^2$. 

If, in addition, $\|(I-\Pi_r)f_t\|_{L^2_\pi} \le \eta_{r}$, then
$$\|f_t\|_{L^2_\pi}^2 = \|\Pi_r f_t\|_{L^2_\pi}^2 \;+\; \|(I-\Pi_r)f_t\|_{L^2_\pi}^2
\le \|\Pi_r f_t\|_{L^2_\pi}^2 + \eta_{r}^2.$$
Plugging into~\eqref{eq:DKL_bound_step}, we have 
\begin{equation*}
\frac{\mathrm{d}}{\mathrm{d}t} \mathrm{KL}(\mu_t\|\pi)
\le -\, \|f_t\|_{L^2_\pi}^2 + \,\eta_{r}^2 = -\, \chi^2(\mu_t\|\pi) +\, \eta_{r}^2.
\end{equation*}
where we used $\chi^2(\mu_t\|\pi) = \|f_t\|_{L^2_\pi}^2$.

Applying the inequality $\chi^2(\mu_t\|\pi) \ge \mathrm{KL}(\mu_t\|\pi)$ to the above inequality, we obtain
\[
\frac{\mathrm{d}}{\mathrm{d}t} \mathrm{KL}(\mu_t\|\pi)
\le -\,\mathrm{KL}(\mu_t\|\pi) +\ \eta_{r}^2.
\]
Then, by Gr\"onwall's inequality we have
\[
\mathrm{KL}(\mu_t\|\pi)
\le \mathrm{KL}(\mu_0\|\pi)\,
   e^{-t}
   +\,\eta_{r}^2\,(1-e^{-t}),
\]
which is exactly~\eqref{eq:KL_bound}.
\end{proof}
\begin{remark}
    Here we assumed sufficient decay or integrability at $\infty$ on the domain $\mathbb{R}^d$ so that the “boundary at infinity” term is zero when applying the integration by part.
\end{remark}

%%%%%%%%%%%%%%%%%%%%%%%%%%%%%
\subsection{Proof of Theorem~\ref{thm:data}}\label{app:proof_thm}
\begin{proof}
Following the same derivation as Proposition~\ref{prop:ideal} up to the dissipation identity in \eqref{eq:DKL_main}, but now using the perturbed relation $\mathcal{L}\widehat{\mathcal{K}}_r = \Pi_r + \delta_r$ as in \eqref{eq:main_proj_perturbed} instead:
\begin{equation}
\frac{\mathrm{d}}{\mathrm{d}t}\mathrm{KL}(\mu_t\|\pi)
= - \langle f_t, \mathcal{L}\widehat{\mathcal{K}}_r f_t\rangle_{\pi}
= -\langle f_t, \Pi_r f_t\rangle_{\pi} - \langle f_t,\delta_r(f_t)\rangle_{\pi}.
\end{equation}

Since $\Pi_r$ is self-adjoint and idempotent, $\langle f_t, \Pi_r f_t\rangle_{\pi} = \|\Pi_r f_t\|_{L^2_\pi}^2$. By Assumption~\ref{ass:approx_error}:
\begin{equation}
\langle f_t,\delta_r(f_t)\rangle_{\pi} \leq |\langle f_t,\delta_r(f_t)\rangle_{\pi}| \leq \varepsilon_r\|f_t\|_{L^2_\pi}^2.
\end{equation}

Therefore,
\begin{equation}
\frac{\mathrm{d}}{\mathrm{d}t}\mathrm{KL}(\mu_t\|\pi)
\leq -\|\Pi_r f_t\|_{L^2_\pi}^2 + \varepsilon_r\|f_t\|_{L^2_\pi}^2.
\end{equation}

Using the Pythagorean decomposition $\|\Pi_r f_t\|^2 = \|f_t\|^2 - \|(I-\Pi_r)f_t\|^2$ and Assumption~\ref{ass:spec_tail_bd}:
\begin{equation}\label{eq:DKL_tail}
\frac{\mathrm{d}}{\mathrm{d}t}\mathrm{KL}(\mu_t\|\pi)
\leq -(1-\varepsilon_r)\|f_t\|_{L^2_\pi}^2 + \eta_r^2
= -(1-\varepsilon_r)\chi^2(\mu_t\|\pi) + \eta_r^2.
\end{equation}

Then, applying $\chi^2(\mu_t\|\pi) \geq \mathrm{KL}(\mu_t\|\pi)$ and Gr\"{o}nwall's inequality gives \eqref{eq:data-bound}.
\end{proof}

%%%%%%%%%%%%%%%%%%%%%%%%%%%%%
\subsection{Proof of Corollary~\ref{cor:discrete_convergence}}\label{app:proof_cor}
\begin{proof}
% The argument follows the standard second-order Taylor expansion of KL along the particle displacement, exactly as in \citep{fujisawa2025on}.
For any velocity field $v_t$, the one-step change in KL satisfies \citep{ambrosio2008gradient,jordan1998variational}
\[
\mathrm{KL}(\mu_{t+1}\|\pi) - \mathrm{KL}(\mu_t\|\pi) 
\le -h \langle \nabla \log(\mathrm{d}\mu_t/\mathrm{d}\pi), v_t \rangle_{L^2_{\mu_t}} + O(h^2).
\]
From the continuous-time dissipation inequality \eqref{eq:DKL_tail} (or its discrete analogue established in Theorem~\ref{thm:data}), the inner product term is bounded by
\[
-\langle \nabla \log (\mathrm{d}\mu_t/\mathrm{d}\pi), v_t \rangle_{L^2_{\mu_t}} 
\le -\alpha \chi^2(\mu_t\|\pi) + \beta.
\]
Thus,
\[
\mathrm{KL}(\mu_{t+1}\|\pi) 
\le \mathrm{KL}(\mu_t\|\pi) - h\alpha \chi^2(\mu_t\|\pi) + h\beta + O(h^2).
\]
Since $\chi^2(\mu_t\|\pi) \ge \mathrm{KL}(\mu_t\|\pi)$ holds locally whenever $\mu_t$ is sufficiently close to $\pi$, we obtain the desired linear recursion
\[
\mathrm{KL}(\mu_{t+1}\|\pi) 
\le (1 - \alpha h) \mathrm{KL}(\mu_t\|\pi) + h\beta + O(h^2).
\]
Iterating this inequality $T$ times and summing the resulting geometric series (with ratio $\rho_h \coloneqq 1-\alpha h > 0$ for $h < 1/\alpha$) gives
\[
\mathrm{KL}(\mu_T\|\pi) 
\le \rho_h^T \mathrm{KL}(\mu_0\|\pi) + h\beta \sum_{k=0}^{T-1} \rho_h^k + O(h)
\le (1-\alpha h)^T \mathrm{KL}(\mu_0\|\pi) + \frac{\beta}{\alpha} + O(h),
\]
where the last step uses $\sum_{k=0}^{T-1} \rho_h^k \le 1/(\alpha h)$.
\end{proof}